\documentclass[10pt]{article} 
\usepackage[accepted]{tmlr}

\usepackage[colorlinks,citecolor=blue,urlcolor=blue,linkcolor=blue,linktocpage=true]{hyperref}
\usepackage{url}
\usepackage{fullpage}
\usepackage{natbib}
\usepackage{algorithm}
\usepackage{algorithmicx}
\usepackage[noend]{algpseudocode}
\usepackage{amsmath, amsfonts, amssymb, amsthm, amsbsy, amscd, bm, bbm,mathrsfs}    \usepackage{graphicx}
\usepackage{xfrac}

\usepackage{xcolor,colortbl}
\definecolor{LightCyan}{rgb}{0.88,1,1}
\newcolumntype{a}{>{\columncolor{LightCyan}}c}

\graphicspath{{figs/}}
\usepackage{subcaption}
\usepackage{cleveref}
\usepackage{enumitem}
\usepackage{tikz}

\setlist{leftmargin=5mm}
\usepackage{titlesec}
\usepackage{wrapfig}
\usepackage{multirow}

\newcommand{\w}{\mathbf{w}}
\newcommand{\x}{\bm{x}}
\newcommand{\p}{\bm{p}}
\newcommand{\g}{\bm{g}}

\newcommand{\R}{\mathbb{R}}
\renewcommand{\P}{\mathbb{P}}
\newtheorem{theorem}{Theorem}

\newtheorem{corollary}{Corollary}

\newtheorem{fact}{Fact}

\theoremstyle{definition}
\newtheorem{definition}{Definition}
\theoremstyle{remark}
\newtheorem{remark}{Remark}

\usepackage{adjustbox}

\title{Differentially Private Optimizers Can Learn Adversarially \\Robust Models}

\author{\name Zhiqi Bu$^*$ \email woodyx218@gmail.com \\
      \AND
      \name Yuan Zhang\thanks{Equal contribution.} \email ewyuanzhang@gmail.com \\
}

\def\month{11}  
\def\year{2023} 
\def\openreview{\url{https://openreview.net/forum?id=o8VgRNYh6n}} 

\begin{document}

\maketitle
\begin{abstract}
Machine learning models have shone in a variety of domains and attracted increasing attention from both the security and the privacy communities. One important yet worrying question is: Will training models under the differential privacy (DP) constraint have an unfavorable impact on their adversarial robustness? While previous works have postulated that privacy comes at the cost of worse robustness, we give the first theoretical analysis to show that DP models can indeed be robust and accurate, even sometimes more robust than their naturally-trained non-private counterparts. We observe three key factors that influence the privacy-robustness-accuracy tradeoff: (1) hyper-parameters for DP optimizers are critical; (2) pre-training on public data significantly mitigates the accuracy and robustness drop; (3) choice of DP optimizers makes a difference. With these factors set properly, we achieve 90\% natural accuracy, 72\% robust accuracy ($+9\%$ than the non-private model) under $l_2(0.5)$ attack, and 69\% robust accuracy ($+16\%$ than the non-private model) with pre-trained SimCLRv2 model under $l_\infty(4/255)$ attack on CIFAR10 with $\epsilon=2$. In fact, we show both theoretically and empirically that DP models are Pareto optimal on the accuracy-robustness tradeoff. Empirically, the robustness of DP models is consistently observed across various datasets and models. We believe our encouraging results are a significant step towards training models that are private as well as robust.
\end{abstract}

\section{Introduction}
Machine learning models trained on large amount of data can be vulnerable to privacy attacks and leak sensitive information. For example, \citet{carlini2021extracting} shows that attackers can extract text input from the training set via GPT2 \citep{radford2019language}, that contains private information such as address, phone number, name and so on; \citet{zhu2019deep} shows that attackers can recover both the image input and the label from gradients of ResNet \citep{he2016deep} trained on CIFAR100 \citep{krizhevsky2009learning}, SVHN \citep{netzer2011reading}, and LFW \citep{huang2008labeled}.

To protect against the privacy risk rigorously, the differential privacy (DP) is widely applied in various deep learning tasks \citep{abadi2016deep,mcmahan2017learning,bu2020deep,nori2021accuracy,li2021large}, including but not limited to computer vision, natural language processing, recommendation system, federated learning and so on. At high level, the privacy is protected via DP optimizers such as DP-SGD and DP-Adam, while allowing the models to remain highly accurate. In other words, the privacy concerns have been largely alleviated by switching from regular optimizers to DP ones.

An equally important concern from the security community is that, many models such as deep neural networks are known to be vulnerable against adversarial attacks. This robustness risk can be severe when the attackers can successfully fool models to make the wrong prediction, through modifying the input data by a negligible amount. An example from \citet{robustness} shows that a strong ResNet50 trained on ImageNet \citep{deng2009imagenet} with 76.13\% accuracy can degrade to 0.00\% accuracy, even if the input image is merely perturbed by 4/255 at each pixel.
 
However, at the intersection of these two concerns, previous works have empirically observed an upsetting privacy-robustness tradeoff under some scenarios, implying the implausibility of achieving both robustness and privacy at the same time. In \citet{song2019privacy} and \citet{mejia2019robust}, adversarially trained models are shown to be more vulnerable to privacy attacks, such as the membership inference attack, than naturally trained models. In \citet{tursynbek2020robustness} and \citet{boenisch2021gradient}, DP trained models were more vulnerable to robustness attacks than naturally trained models on MNIST and CIFAR10. This leads to the following concern:
\begin{center}
Does DP optimization necessarily lead to less adversarially robust models?
\end{center}

On the contrary, we show that DP models can be adversarially robust, sometimes even more robust than the naturally trained models. Indeed, we illustrate that DP models can be Pareto optimal, so that any model with higher accuracy than DP ones must have worse robustness. We observe that:
\begin{enumerate}
\item DP training itself does not worsen adversarial robustness in comparison to the natural training.
\item The robustness is largely affected by the DP optimization hyper-parameters $(R,\eta)$, where different hyper-parameters are equally privacy-preserving but significantly different in accuracy and robustness.
\item DP optimization hyper-parameters $(R,\eta)$ achieving the best accuracy (which can be much less robust than the natural training) is different to those achieving the best robustness.
\item The robustness is less affected by the privacy budget (e.g. \Cref{tab:multi-atk-cifar10_rob} and \Cref{tab:resnet18 general attacks}). Even when we consider non-DP optimization (i.e. $\epsilon=\infty$, no noise but with clipping), the models could be comparably or more robust than naturally trained models (see the blue boxplots in \Cref{fig:dpsgd-intercept-r-boxplot}).
\end{enumerate}

In sharp contrast to the empirical nature of previous arts, we enhance our understanding about the adversarial robustness of DP models from a theoretical angle. Our analysis shows that DP classifiers without adversarial training may in fact be the most adversarially robust classifier. Motivated by our theoretical analysis, we claim that the hyperparameter tuning is vital to successfully learning a robust and private model, where the optimal choice is to use small clipping norms and large learning rates. This is interesting as such a hyperparameter choice is also observed to be the most effective in learning highly accurate models under DP \citep{li2021large,kurakin2022toward,de2022unlocking}. In fact, using a small clipping norm is equivalent to normalizing all per-sample gradients, which allows a clear demonstration in \Cref{tab:auto} that the optimal hyperparameter for natural accuracy is different to that for robust accuracy (only $\eta$ is present because $R$ is absorbed in the automatic DP-SGD \citep{bu2022automatic}).

\begin{table}[!htb]
    \centering
    \caption{Robust and natural accuracy are achieved by different hyperparameter $\eta$ on CIFAR10. Same setting as in \citet{tramer2020differentially}, under $(\epsilon,\delta)=$(2,1e-5) and attacked by 20 steps of $l_\infty(2/255)$ PGD.}
    \resizebox{0.9\linewidth}{!}{
    \begin{tabular}{c|c|cccccccc}
        \hline
        \multirow{3}{*}{DP}& learning rate $\eta$ & $2^{-8}$& $2^{-7}$& $2^{-6}$& $2^{-5}$& $2^{-4}$& $2^{-3}$& $2^{-2}$& $2^{-1}$ \\\cline{2-10}
        &natural accuracy & 84.36&87.33&89.45&90.76&91.76&92.50&92.54&\textbf{92.70}\\
        & robust accuracy & 75.45&78.92&\textbf{81.03}&80.96&78.87&72.77&58.29&26.97\\\hline\hline
        \multirow{3}{*}{non-DP}& learning rate $\eta$ & $2^{-5}$ & $2^{-4}$& $2^{-3}$& $2^{-2}$& $2^{-1}$& $2^{0}$& $2^{1}$& $2^{2}$ \\\cline{2-10}
        & natural accuracy & 94.23& 94.31&94.38&94.46&\textbf{94.60}&94.39&94.21&93.84\\
        & robust accuracy &79.32& \textbf{79.57}&74.18&66.20&55.87&46.00&41.75&45.21\\
        \hline
    \end{tabular}
    }
    \label{tab:auto}
\end{table}

\vspace{-0.2cm}
Additionally, we advocate pretraining the model and selecting proper optimizers for DP training, which allow us to max out the performance on MNIST \citep{lecun1998gradient}, Fashion MNIST\citep{xiao2017fashion}, CIFAR10\citep{krizhevsky2009learning}, and CelebA\citep{liu2015faceattributes}.
\begin{remark}
Most of this work does not use adversarial training (except in \Cref{tab:cifar with madry} and \Cref{tab:cifar with madry l2}) and should be distinguished from the certified robustness \citep{lecuyer2019certified}, which does not guarantee DP in a per-sample sense rather than use the mathematical tools from DP in a per-pixel way.
\end{remark}

\section{Preliminaries}
\textbf{Notation.} We use $f: \mathcal{X} \rightarrow \mathcal{Y}$ to denote the model mapping from data space $\mathcal{X}$ to label space $\mathcal{Y}$. We denote the datapoints as $\{\x_i\}\in\R^d$ and the labels as $\{y_i\}$, following i.i.d. from the distributions $\x$ and $y$ respectively. We denote the gradient of the $i$-th sample at step $t$ as $g_t(\x_i,y_i;\w,b)$, where $\w$ is the weights and $b$ is the bias of model $f$.

To start, we introduce the definition of DP, particularly the $(\epsilon,\delta)$-DP \citep{dwork2014algorithmic}.
\begin{definition}\label{def:DP}
A randomized algorithm $M$ is $ (\varepsilon, \delta)$-DP if for any neighboring datasets $ S, S^{\prime} $ that differ by one arbitrary sample, and for any event $E$, it holds that
\begin{align}
 \mathbb{P}[M(S) \in E] \leqslant \mathrm{e}^{\varepsilon} \mathbb{P}\left[M\left(S^{\prime}\right) \in E\right]+\delta.
\end{align}
\end{definition}

In words, DP guarantees in the worst case that adding or removing one single datapoint $(\x_i,y_i)$ does not affect the model much, as quantified by the small constants $(\epsilon,\delta)$. Therefore DP limits the information possibly leaked from such datapoint.



In deep learning, DP is guaranteed by privatizing the gradient in two steps: (1) the per-sample gradient clipping\footnote{We use $C_R(g_t(x_i)):=g_t\cdot\min\{R/||g_t||_2,1\}$ as in \citet{abadi2016deep} to denote the gradient clipping, after which each per-sample gradient has norm $\leq R$. Note that in \Cref{tab:auto} we use the automatic clipping such that $C_R(g_t(x_i)):=g_t/||g_t||_2$.} (specified by the clipping norm $R$, to bound the sensitivity of $\sum \g_t(x_i)$); (2) the random noising (specified by the noise multiplier $\sigma_\text{DP}$, to randomize the outcome so that each sample's contribution is indistinguishable; $\sigma_\text{DP}$ can be determined by the privacy accountants \cite{dwork2014algorithmic,abadi2016deep}). From an algorithmic viewpoint, DP training simply applies any optimizer on the private gradients instead of on the regular gradients.

\vspace{-0.6cm}
\begin{align}
&\text{Non-DP training on regular gradient:} & {\textstyle\sum}_i \g_t(\x_i,y_i)
\\
&\text{DP training on private gradient:} & {\textstyle\sum}_i C_R(\g_t(\x_i,y_i))+\sigma_\text{DP} R\cdot\mathcal{N}(\bm 0,\bm I)
\label{eq:private grad}
\end{align}

\vspace{-0.2cm}
We are interested in the adversarial robustness of models trained by DP optimizers. To be sure, we consider the adversarially robust classification error and the natural classification error as
\begin{align}
\mathcal{R}_{\gamma}(f):=\P(\exists ||\p||_\infty<\gamma, \text{ s.t. } f(\x+\p)\neq y),\quad 
\mathcal{R}_{0}(f):=\P(f(\x)\neq y).
\label{eq:robust error}    
\end{align}
where $\p\in\R^d$ is the adversarial perturbation, $\gamma$ is the attack magnitude, and $l_\infty$ (and $l_2$) attack is considered. Notice that when $\gamma=0$, the robust error in \eqref{eq:robust error} reduces to the natural error.


\section{Theoretical Analysis on Linear Classifiers}

To theoretically understand the adversarial robustness of DP learning, we study the robustness of the DP and non-DP linear models on a binary classification problem. We consider a mixed Gaussian distribution, where the positive class $y=+1$ has a larger variance (i.e. it is more difficult to be classified correctly\footnote{The fact that larger variance indicates lower intra-class accuracy is proven by \citet[Theorem 1]{xu2021robust}.}) than the negative class $y=-1$:
\begin{align}
  \bm x \sim\left\{\begin{array}{ll}\mathcal{N}\left(\bm\theta_d, K^2\sigma^{2} \mathbf{I}_d\right) & \text { if } y=+1 \\ \mathcal{N}\left(-\bm\theta_d, \sigma^{2}\mathbf{I}_d\right) & \text { if } y=-1\end{array}\right.
  \label{eq:x distribution}
\end{align}
where $y \overset{\text{unif}}{\sim}\{-1,+1\}$, $\bm\theta_d=(\theta,\cdots,\theta)\in\R^d$, $\sigma>0$ and $K>1$.

\begin{figure}[H]
    \centering
\includegraphics[width=0.3\linewidth]{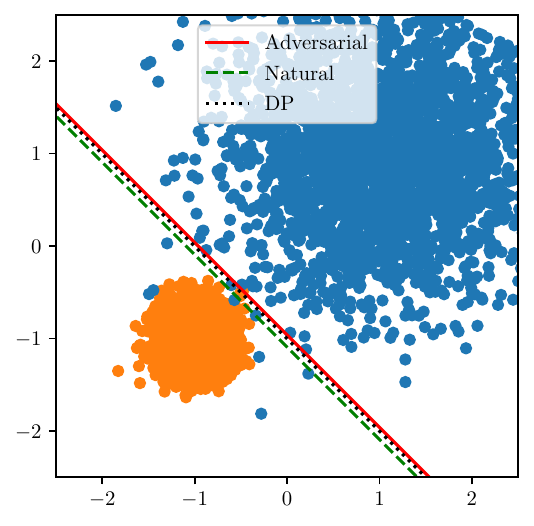}
\includegraphics[width=0.33\linewidth]{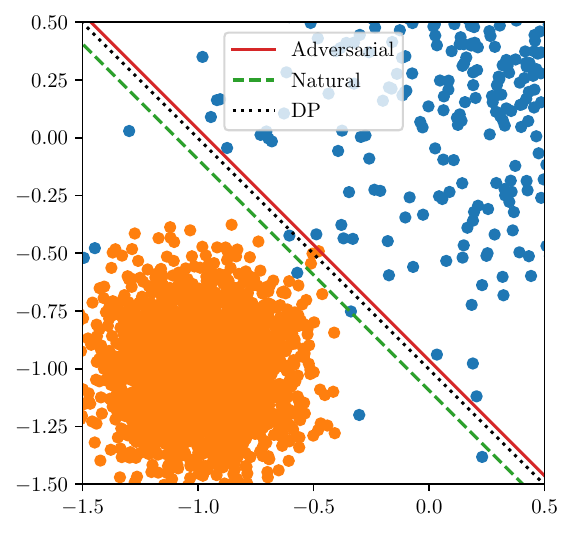}
    \vspace{-0.3cm}
    \caption{Decision boundaries of linear classifiers for \eqref{eq:x distribution}, $K=4,\sigma=0.2,\theta=1$.}
    \label{fig:decision-boundary-logistic-gaussian}
\vspace{-0.4cm}
\end{figure}

The setting\footnote{This setting is also studied in \citet{xu2021robust}, which focuses on the robustness-fairness tradeoff, and thus is different to our interest.} in \eqref{eq:x distribution} is analyzable because of the data symmetry along the diagonal axis $\mathbb{E}x_1=\cdots=\mathbb{E}x_d$, as illustrated in \Cref{fig:decision-boundary-logistic-gaussian}. We will show that this symmetry leads to an explicit decision hyperplane of linear classifiers in \Cref{thm:robust classifier}, which further characterizes the strongest adversarial perturbation $\p^*\equiv\arg\sup_{\|\p\|_\infty<\gamma}\P(f(\x+\p)\neq y)$ explicitly. 

In the following analysis, we focus on the linear classifiers $f(\bm x;\w,b)=\text{sign}(\sum_{j=1}^d w_j x_j+b)$ with weights $w_j$ and bias $b$.

\begin{remark}
Within the family of linear classifiers, by the symmetry of data in \eqref{eq:x distribution}, it can be rigorously shown by \eqref{eq: weights are ones} that the optimal weights with respect to the natural and robust errors are always $w_1=\cdots=w_d$, which matches our empirical calculation on DP models. 
That is, the weights do not distinguish between the robust and natural models. Consequently, the key to the adversarial robustness lies in the intercept $b$, which is analyzed in the subsequent sections.
\end{remark} 

\subsection{Optimal Robust and Natural Linear Classifiers}

We start by reviewing the robust error of robust classifier and the explicit formula of its intercept.

\begin{theorem}[Extended from Theorem 2 in \citet{xu2021robust}]\label{thm:robust classifier}
For data distribution $(\x,y)$ in \Cref{eq:x distribution} and under the $\gamma$ attack magnitude, we define the optimal robust linear classifier as
$$
  f_{\gamma}=\underset{f \text{ is linear}}{\arg\min}\ \P(\exists ||\p||_\infty<\gamma, \text{ s.t. } f(\x+\p)\neq y)=\underset{f \text{ is linear}}{\arg\min}\ \mathcal{R}_{\gamma}(f).
$$

\vspace{-0.4cm}
The optimal robust error is
\vspace{-0.1cm}
\begin{align*}
\mathcal{R}_{\gamma}\left(f_\gamma\right)=\frac{1}{2}\Phi\left(B(K,\gamma)-K \sqrt{B(K,\gamma)^{2}+q(K)}\right) +\frac{1}{2}\Phi\left(-K B(K,\gamma)+\sqrt{B(K,\gamma)^{2}+q(K)}\right),
\end{align*}

\vspace{-0.4cm}
where $\Phi$ is the cumulative distribution function of standard normal, $B(K,\gamma)=\frac{2}{K^{2}-1} \frac{\sqrt{d} (\theta-\gamma)}{\sigma} $ and $ q(K)=\frac{2 \log K}{K^{2}-1} $. Furthermore, by the symmetry of the data distribution, we have
\begin{align}
1,\cdots,1,b_{\gamma}&=\underset{\w,b}{\arg\min} \mathcal{R}_{\gamma}(f(\cdot;\w,b)),
\label{eq: weights are ones}
\\
b_{\gamma}&=\frac{K^{2}+1}{K^{2}-1} d (\theta-\gamma)-K \sqrt{\frac{4d^{2} (\theta-\gamma)^{2}}{\left(K^{2}-1\right)^{2}}+d \sigma^{2} q(K)}.
\label{eq:b_gamma}
\end{align}
\end{theorem}

\Cref{thm:robust classifier} gives the closed form of the optimal robust classifier $f_\gamma$, or equivalently its intercept $b_\gamma$, and the optimal robust error. The special case of natural error can be easily recovered by setting $\gamma=0$:
$$1,\cdots,1,b_{0}=\arg\min_{\w,b} \mathcal{R}_0(f(\cdot;\w,b)).$$

\begin{wrapfigure}{r}{0.46\textwidth}
\vspace{-1cm}
    \centering
    \includegraphics[width=\linewidth]{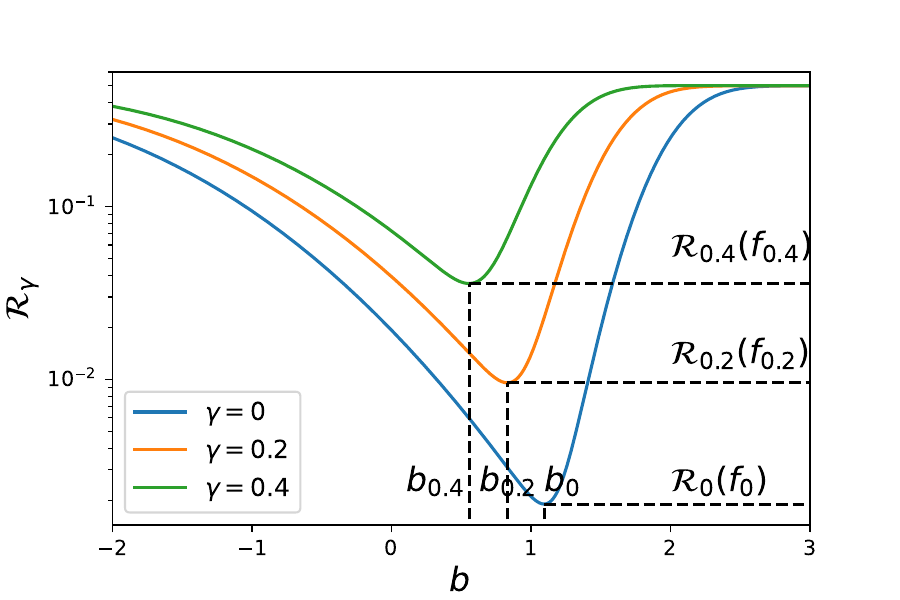}
    \vspace{-0.7cm}
    \caption{Intercepts and robust/natural accuracy under \eqref{eq:x distribution}, with $K=4,\sigma=0.2,\theta=1$.}
    \label{fig:b and R}
    \vspace{-1.5cm}
\end{wrapfigure}

We know for sure from \Cref{thm:robust classifier} that there exists a tradeoff between robustness and accuracy: it is impossible for the natural classifier $f_0$ to be optimally robust or the robust classifier $f_\gamma$ to be optimally accurate, since $b_0\neq b_\gamma$ (c.f. \Cref{fig:b and R}).

\begin{fact}\label{fact:b decreasing}
$b_\gamma$ in \eqref{eq:b_gamma} is strictly decreasing in $\gamma$, ranging from $b_0$ to $-\infty$.
\end{fact}

\begin{proof}[Proof of \Cref{fact:b decreasing}]
Proof 5 in \citet{xu2021robust} shows that $\frac{d b_\gamma}{d\gamma}\leq-\frac{K-1}{K+1}d<0$, thus $b_\gamma$ is strictly decreasing in $\gamma$. Therefore, the range of $b_\gamma$ is $(b_{\infty},b_0]$. Finally, we note that $b_{\gamma} < \frac{K^2+1}{K^2-1}(\theta-\gamma)$, hence $b_{\infty} = -\infty$.
\end{proof}

\subsection{Adversarially Robust Errors of Private Linear classifiers}
Now we analyze the robust error of DP classifiers, which requires a different analysis from \Cref{thm:robust classifier} because $\langle 1\rangle$ DP modifies the optimization instead of the objective function, unlike $\mathcal{R}_{\gamma}$ which modifies the natural error $\mathcal{R}_{0}$; $\langle 2\rangle$ consequently,  deriving the DP parameters $\w$ and $b$ is much more difficult: we need to solve ${\arg\min}_{\w,b}\mathcal{R}_{\gamma}(f(\cdot;\w,b))$ under the additional constraint (imposed by DP) that \eqref{eq:private grad} equals 0 in expectation. While it is possible to simplify the problem, for example, by analyzing the deterministic gradient flow \cite{bu2021convergence} or by formulating the actual objective that DP-SGD optimizes \cite{song2021evading}, this is beyond the scope of this work.

We consider a specific linear classifier $f(\cdot;\mathbf{1},b)$ where only the intercept $b$ is learned and privatized\footnote{Note that DP is only required on trainable parameters that are learned from data; otherwise no data privacy can be leaked. Therefore this specific classifier is guaranteed to be DP.}. This is known as the bias term fine-tuning \citep{zaken2021bitfit,bu2022differentially}, a popular approach in fine-tuning the DP or standard neural networks.

For this classifier and any $b$, the robust error ($\gamma\neq 0$) and the natural error ($\gamma= 0$) are:
\begin{align}
\mathcal{R}_{\gamma}(f)=& \P(\exists\|\p\|_\infty \leq \gamma \text { s.t. } f(\x+\p) \neq y) = \max _{\|\p\|_\infty \leq \gamma} \P(f(\x+\p) \neq y) 
\nonumber
\\=& \frac{1}{2} \P(f(\x+\bm\gamma_d) \neq-1 \mid y=-1) +\frac{1}{2} \P(f(\x-\bm\gamma_d) \neq+1 \mid y=+1) 
\nonumber
\\=& \frac{1}{2}\P\left(\sum_{i=1}^{d}w_j\left(x_j+\gamma\right)+b>0 \mid y=-1\right) +\frac{1}{2}\P\left(\sum_{i=1}^{d}w_j\left(x_j-\gamma\right)+b<0 \mid y=+1\right)
\nonumber
\\
=&\frac{1}{2}\Phi\left(-\frac{\sqrt{d}(\theta-\gamma)}{\sigma}+\frac{1}{\sqrt{d} \sigma} \cdot b\right) 
+\frac{1}{2}\Phi\left(-\frac{\sqrt{d}(\theta-\gamma)}{K \sigma}-\frac{1}{K \sqrt{d} \sigma} \cdot b\right)
\label{eq:robust error of any intercept}
\end{align}
where $\bm\gamma_d\equiv(\gamma,\cdots,\gamma)$. With \eqref{eq:robust error of any intercept}, we can analyze the robust and natural errors for any intercept $b$ (private or not, robust or natural) and any attack magnitude $\gamma$.

Our next result answers the following question: fixing a DP classifier $f_\text{DP}:=f(\cdot;\mathbf{1},b_\text{DP})$, or equivalently its intercept $b_\text{DP}$, under which attack magnitude is the classifier robust? We show that, 
although $b_\text{DP}$ is not available in the closed form, 
it is possible for some attack magnitude $\gamma^*$ that $b_\text{DP}=b_{\gamma^*}$, and thus the DP classifier is the most robust classifier among all. 

\begin{theorem}\label{thm:DP is most robust}
For data distribution $(\x,y)$ in \Cref{eq:x distribution} and for any $b_\textup{DP}<b_0$, there exists $\gamma^*>0$ such that $b_{\gamma^*}=b_\textup{DP}$, and therefore
$$\underset{f\text{ is linear}}{\min}\mathcal{R}_{\gamma^*}(f)\equiv\mathcal{R}_{\gamma^*}(f_{\gamma^*})=\mathcal{R}_{\gamma^*}(f_\textup{DP}).$$
\end{theorem}
\begin{proof}[Proof of \Cref{thm:DP is most robust}]
By \Cref{fact:b decreasing}, $b_{\gamma}-b_\text{DP}$ is decreasing in $\gamma$, ranging from $b_0-b_\text{DP}$ to $-\infty$. By the intermediate value theorem, there exists $\gamma^*>0$ such that $b_{\gamma^*}=b_\text{DP}$, i.e. $f_\text{DP}=f_{\gamma^*}(\cdot;\mathbf{1},b_{\gamma^*})$.
\end{proof}

By \Cref{thm:DP is most robust}, as long as the DP intercept is sufficiently small, the DP classifier must be the most robust under some attack magnitude, among all linear classifiers. We visualize in \Cref{fig:dpsgd-intercept-r-boxplot} at $\gamma^*=0.075$, that indeed $b_\text{DP}\approx b_{\gamma^*}$ (grey solid line).

\begin{figure}[!htb]
    \centering
    \includegraphics[width=0.7\linewidth]{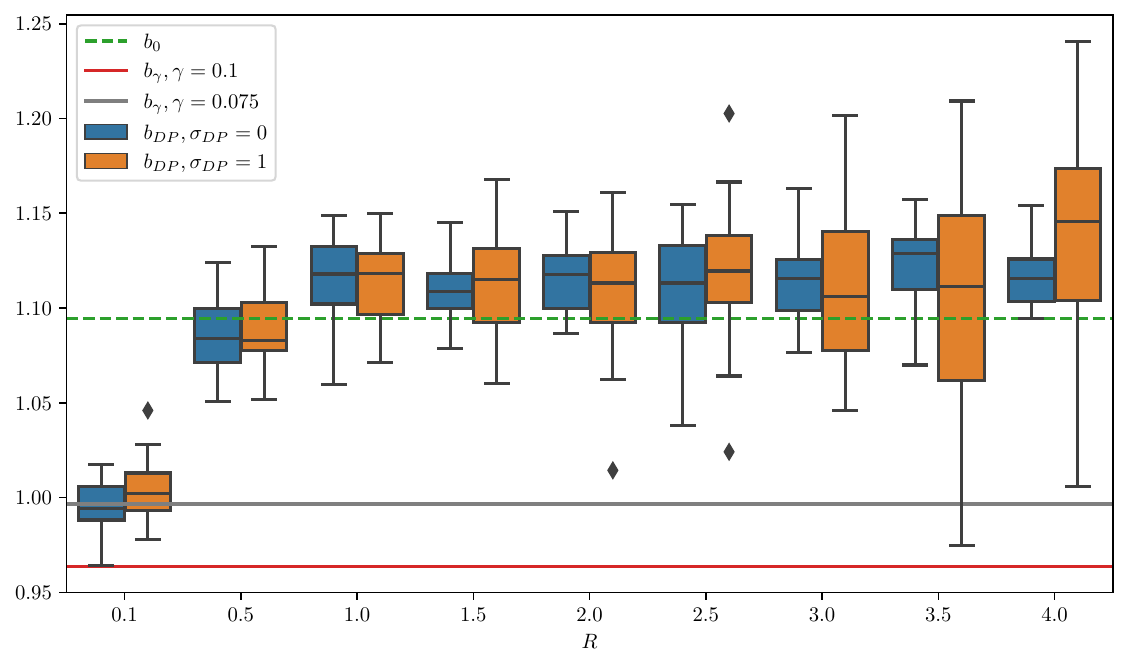}
    \vspace{-0.3cm}
    \caption{Intercepts (decision boundaries) for \eqref{eq:x distribution}, same setup as \Cref{fig:b and R}. For DP classifiers, we use DP-SGD with $\eta=8$, epochs=50, batch size=1000, sample size=10000, $(\epsilon,\delta)=(15,1e-4)$. }
    \label{fig:dpsgd-intercept-r-boxplot}
\end{figure}

To validate the condition of \Cref{thm:DP is most robust}, we now demonstrate the achievability of $b_\text{DP}<b_0$ as a result of the per-sample gradient clipping in DP optimizers. In words, we show that the robust intercept is stationary by the DP gradient descent but not so by the regular gradient descent.

We consider the above setting with an attack magnitude $\gamma$ and $\sigma_\text{DP}=0$. We temporarily ignore the noise because it only adds variance to the gradient but does not affect the mean (see the blue and orange boxes in \Cref{fig:dpsgd-intercept-r-boxplot}), and when learning rate is small, $\sigma_\text{DP}$ has little effect on convergence \citep{bu2021convergence}. 

In DP gradient descent, if the clipping norm $R$ is sufficiently small, then each per-sample gradient $\g_t(\x_i,y_i;\mathbf{1},b_\gamma)$ has the same magnitude $R$ after clipping. Therefore the positive samples, pushing the intercept $b$ to increase, can balance with the negative ones that pull $b$ to descrease. Thus $b=b_\gamma$ is a stationary point learnable by DP training, even though it is smaller than $b_0$ by \Cref{fact:b decreasing}. However, in the non-DP gradient descent, $b_\gamma$ is not stationary. This is because the positive class gradient $\sum_i g_t(\x_i,+1;\mathbf{1},b_\gamma)$ is larger than the negative class gradient $\sum_i g_t(\x_i,-1;\mathbf{1},b_\gamma)$, so as to push the decision boundary $b_\gamma$ towards $b_0$, where the natural classifier $f_0$ is defined. We visualize our analysis in \Cref{fig:gradient-dist-gaussian-dptrained}.

\begin{figure}[!htb]
    \vspace{-0.2cm}
    \centering
    \includegraphics[width=0.4\linewidth]{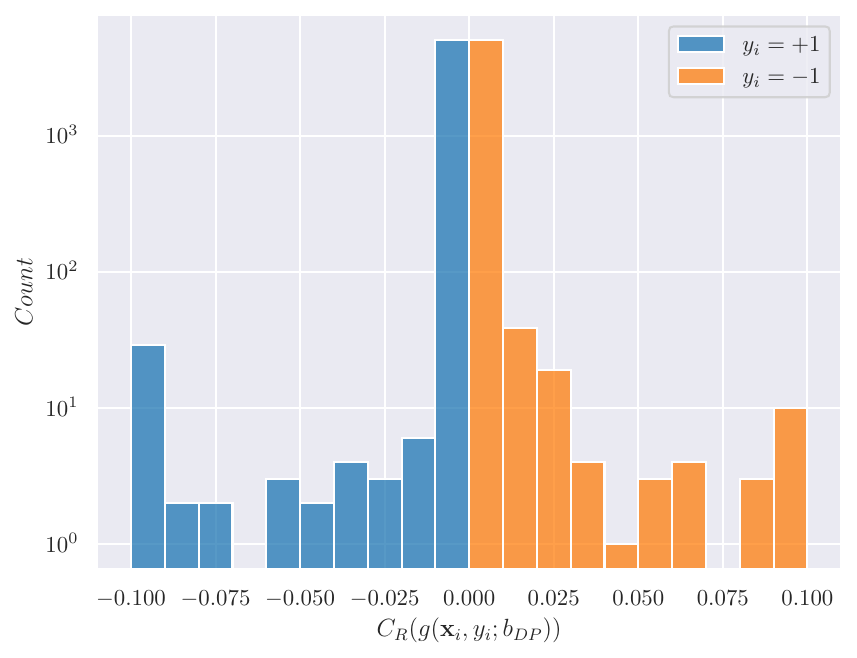}
    \includegraphics[width=0.4\linewidth]{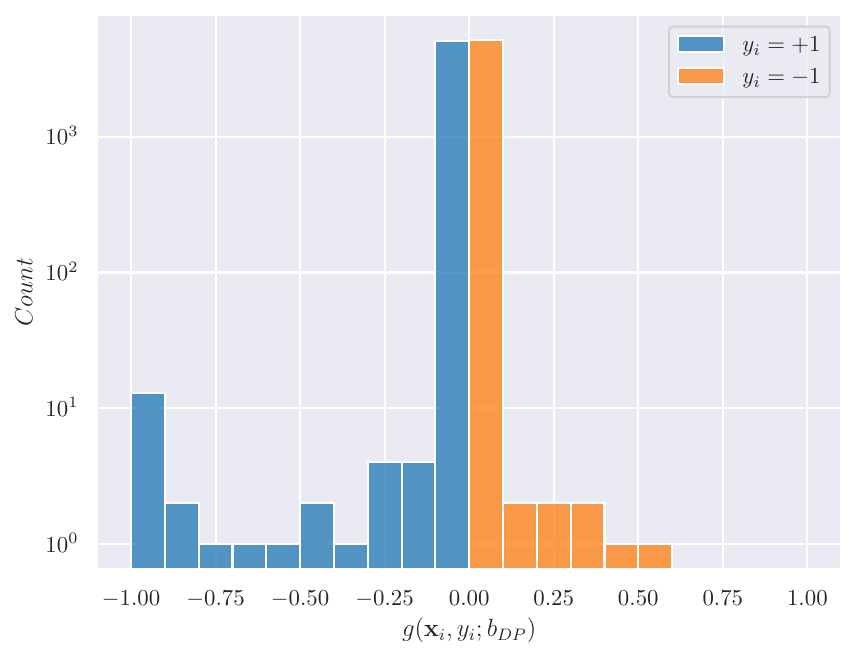}
    \vspace{-0.2cm}
    \caption{Distribution of gradient with clipping (left, which is balanced) and without clipping (right, which is unbalanced) of linear classifiers for \eqref{eq:x distribution}, $K=4,\sigma=0.2,\theta=1$, and clipping norm $R=0.1$.}
    \label{fig:gradient-dist-gaussian-dptrained}
\end{figure}

Next, suppose we only require the DP classifier to be more robust than the natural classifier, without requiring it to be the most robust among all linear classifiers. We can answer the question: fixing the attack magnitude $\gamma$, under which condition is $f_\text{DP}$ more robust than the natural classifier $f_0$? 
\begin{theorem}\label{thm:DP is robust than nat}
Fixing the attack magnitude $\gamma$, if data distribution in \Cref{eq:x distribution} satisfies $\frac{K^2+1}{2K}\gamma<|\theta-\gamma|+|\theta|$, then whenever $b_\gamma<b_\textup{DP}<b_0$, we have
$$\underset{f\text{ is linear}}{\min}\mathcal{R}_\gamma(f)\equiv\mathcal{R}_\gamma(f_\gamma)<\mathcal{R}_\gamma(f_\textup{DP})<\mathcal{R}_\gamma(f_0).$$
Furthermore, any intercept $b$ with better natural accuracy than $b_\textup{DP}$ must have worse robust accuracy:
$$\mathcal{R}_0(f)<\mathcal{R}_0(f_\textup{DP})\Longrightarrow\mathcal{R}_\gamma(f)>\mathcal{R}_\gamma(f_\textup{DP}).$$
\end{theorem}

\Cref{thm:DP is robust than nat} shows that, under some conditions, DP models are more robust than natural models, and cannot be dominated in the Pareto optimal sense. That is, DP linear classifier can be Pareto optimal in terms of the robust and natural accuracy.
We visualize the premise $b_\gamma<b_\textup{DP}<b_0$ in \Cref{fig:dpsgd-intercept-r-boxplot} as well as in \Cref{fig:decision-boundary-logistic-gaussian}, and the result $\mathcal{R}_{\gamma}(f_{\gamma})<\mathcal{R}_\gamma(f_\text{DP})<\mathcal{R}_\gamma(f_0)$ in \Cref{fig:b and R}. 

We emphasize that, our results on the $l_\infty$ attacks is generally extendable to $l_2$ attacks. Put differently, we show that DP models can be adversarially robust and Pareto optimal under both $l_\infty$ and $l_2$ attacks\footnote{In DP deep learning, the clipping is on the gradient level under $l_2$ norm (see Footnote 1), regardless of the $l$ norm in adversarial attacks, which are on the sample level, i.e. on $\x_i$.}.
\begin{corollary}[Extension to $l_2$ attacks]
\label{cor:l2 attacks}
All theorems hold for $l_2$ attacks by changing $\gamma\to\gamma/\sqrt{d}$.
\end{corollary}

\begin{remark}
Theorem 2 gives sufficient and necessary condition for the DP classifier to be more robust than all classifiers at one attack magnitude $\gamma^*$; Theorem 3 gives sufficient but not necessary condition for the DP classifier to be more robust than one classifier (the natural one) at many attack magnitudes.
\end{remark}

\section{Training Private and Robust Neural Networks}
\label{sec:experiments}

In this section, we extend our investigation beyond the linear classifiers in \Cref{thm:robust classifier}, \Cref{thm:DP is most robust}, and \Cref{thm:DP is robust than nat}, and study the robustness of DP neural networks. We emphasize that several state-of-the-art (SOTA) advances are actually achieved by \textit{linear classifiers} within the deep neural networks \citep{mehta2022large,tramer2020differentially}, i.e. by finetuning only the last linear layer of Wide ResNet, SimCLR, and vision transformers. Therefore these advances fall in the same setting as our theorems, although the neural networks are non-linear models in general.

By experimenting with real datasets MNIST, CIFAR10 and CelebA, we corroborate our insights gained from theoretically analyzing the linear models and empirically show that the DP neural networks can be adversarially robust in practice (despite being much more challenging to analyze). We use one Nvidia GTX 1080Ti GPU and the Renyi privacy accountant to calculate the privacy loss.

In summary, we have three key observations. 
\begin{enumerate}
    \item By selecting the hyper-parameters carefully, we can remain exactly the same level of DP but vastly stronger robustness. Especially, we visualize the distinct landscapes of robust accuracy and natural accuracy over $(R,\eta)$, which also depend on the choice of optimizers (see \Cref{app:RMSprop}).
    \item We observe a privacy-accuracy-robustness tradeoff, showing that DP models is Pareto optimal (with or without pre-training), thus extending \Cref{thm:DP is robust than nat} to deep learning.
    \item The robustness of DP models holds for general attacks, including single-step or multi-step (FGSM v.s. PGD), single method or ensemble (PGD v.s. APGD), and $l_\infty$ or $l_2$.
\end{enumerate}

\vspace{0.6cm}
\begin{remark}
Our analysis also implies that DP models can be resilient to data poisoning attacks, since adversarial examples can serve as strong data poisons \citep{fowl2021adversarial}. Such resilience is empirically observed in \citep{yang2022not,hong2020effectiveness}. 
\end{remark}

\vspace{0.5cm}
\subsection{Hyper-parameters are keys to robustness}
\label{sec:small clip regime}
\vspace{0.2cm}
In DP deep learning, the training hyper-parameters can be divided into two categories: some are related to the privacy accounting, including the batch size $B$, the noise multiplier $\sigma_\text{DP}$, the number of iterations $T$; the others are only related to the optimization but not to the privacy, including the clipping norm $R$ and the learning rate $\eta$. That is, changing $(R,\eta)$ can influence the accuracy and the robustness without affecting the DP guarantee.

On one hand, $R$ has to be small to achieve SOTA natural accuracy. Large models such as ResNet and GPT2 are optimally trained at $R<1$, even though the gradient's dimension is of hundreds of millions \citet{kurakin2022toward,li2021large,klause2022differentially,mehta2022large}. In fact, \cite{bu2022automatic} adopts an infinitely small $R=0^+$ to achieve SOTA results, essentially applying per-sample gradient normalization instead of the clipping.

On the other hand, DP training empirically benefits from large learning rate, usually 10 times larger than the non-DP training. This pattern is observed for DP-Adam \citep[Figure 4]{li2021large} and for DP-SGD \citep{kurakin2022toward} over text and image datasets.

Interestingly, we also observe such choice of $(R,\eta)$ performs strongly in the adversarial robustness context (though not the same hyper-parameters). By the ablation study in \Cref{fig:params-matter-celeba-simclr-sgdmtm-loglog} for CIFAR10 and in \Cref{app: ablation} for MNIST, Fashion MNIST and CelebA, it is clear that robust accuracy and natural accuracy have distinctively different landscapes over $(R,\eta)$. We observe that the optimal $(R,\eta)$ should be carefully selected along the diagonal ridge for DP-SGD to obtain high robust and high natural accuracy. Otherwise, even small deviation can lead to a sharp drop in the robustness, despite that the natural accuracy may remain similar (see upper right corner of 2D plots in \Cref{fig:params-matter-celeba-simclr-sgdmtm-loglog}). 

\vspace{0.6cm}
\begin{remark}
From \Cref{fig:params-matter-celeba-simclr-sgdmtm-loglog} (see also \Cref{fig:params-matter-celeba-adam-loglog}, \Cref{fig:optimizers CelebA} and \Cref{fig:params-matter-celeba-adam-loglog22} in the appendix), it is empirically sufficient to set $R\approx 0^+$ and to only tune the learning rate $\eta$ for both robust and natural accuracy, when using DP optimizers such as DP-SGD, DP-Adam, DP-RMSprop and DP-Adagrad. This approach, termed as automatic clipping by \cite{bu2022automatic}, reduces the 2-dimension hyperparameter search to a much cheaper 1-dimension search (c.f. \Cref{tab:auto}).
\end{remark}

\newpage
\begin{figure}[!htb]
\centering
\includegraphics[width=0.36\linewidth]{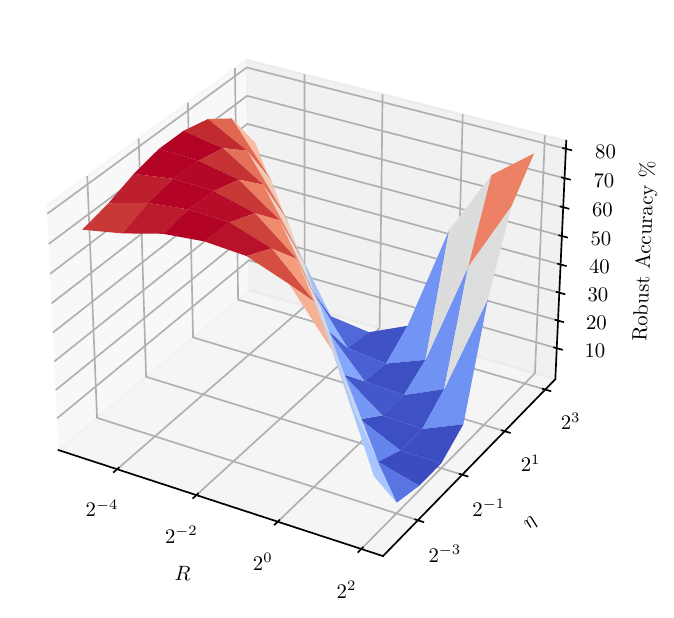}
    \includegraphics[width=0.38\linewidth]{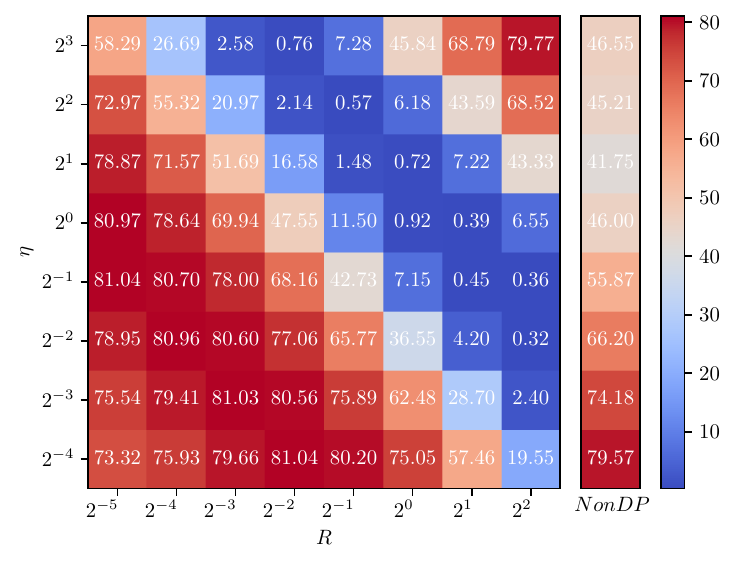} 
    \\\vspace{-0.2cm}
    \includegraphics[width=0.36\linewidth]{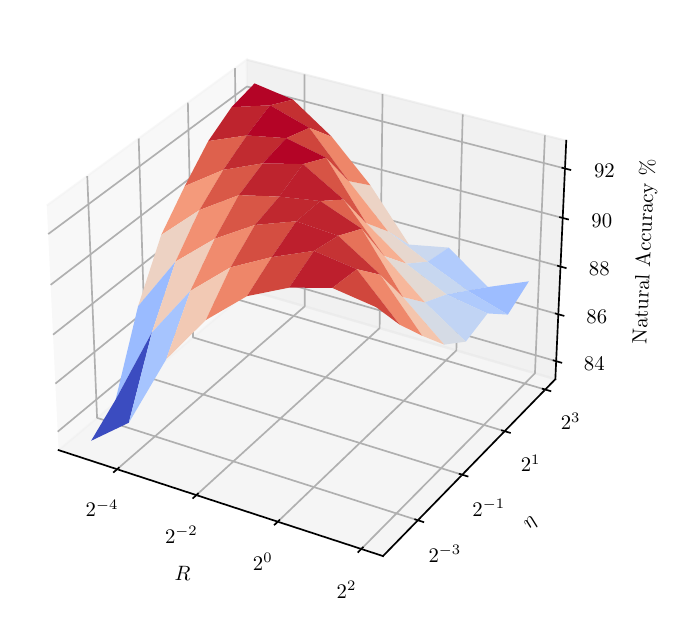}
    \includegraphics[width=0.38\linewidth]{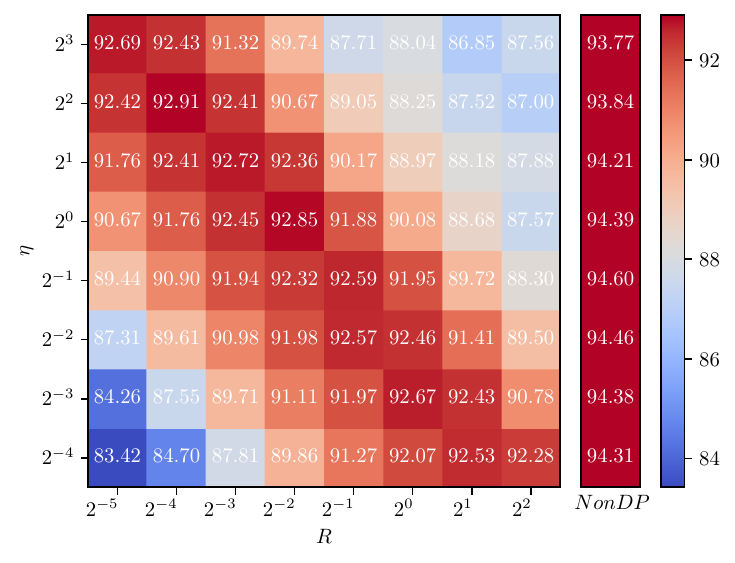}
    \caption{Robust and natural accuracy by $\eta$ and $R$ on CIFAR10. We use the same setting as in \citet{tramer2020differentially}: pretraining SimCLR on ImageNet and then privately training using DP-SGD with momentum=0.9, under $(\epsilon,\delta)=(2,1e-5)$ and attacked by 20 steps of $l_\infty(2/255)$ PGD.} 
    \label{fig:params-matter-celeba-simclr-sgdmtm-loglog}
\end{figure}

Our ablation study demonstrates that the DP neural network, with $81.04\%$ robust accuracy and $89.86\%$ natural accuracy, can be more robust than the most robust of naturally trained networks ($79.59\%$ robust accuracy and $94.31\%$ natural accuracy). If we trade some robustness for the natural accuracy, we can achieve the same level of robustness $(80.20\%)$ at 91.27\% natural accuracy, thus closing the gap between the natural accuracy of DP and non-DP models without sacrificing the robustness.

While \Cref{fig:params-matter-celeba-simclr-sgdmtm-loglog} presents the result of a single attack magnitude, we further study the influence of hyper-parameters under different attack magnitudes, with and without the adversarial training. We illustrate on CIFAR10 the $l_\infty$ attack performance in \Cref{tab:cifar with madry} and the $l_2$ one in \Cref{tab:cifar with madry l2}.

\begin{table}[H]
    \centering
\caption{Natural and robust accuracy of SimCLRv2 \citep{chen2020simple} and ResNet50 \citep{robustness} on CIFAR10 under 20 steps $l_\infty$ PGD attack. Here \textcolor{cyan}{cyan} columns are adversarial training and white columns are natural training. Adversarial training and \textit{robust} hyper-parameters are obtained by grid search over $\eta$ and $R$ against $l_\infty(2/255)$, and \textit{natural} hyper-parameters are adopted from \citet{tramer2020differentially}.
}
\resizebox{\linewidth}{!}{
\setlength\tabcolsep{1.2pt}
\begin{tabular}{c|acc|acc|acc|acc|ac}
\hline
&\multicolumn{12}{c}{SimCLRv2 pre-trained on unlabelled ImageNet}&\multicolumn{2}{|c}{ResNet50}
\\
\hline
 & DP& DP& DP& DP& DP& DP& DP & DP& DP& Non-DP& Non-DP & Non-DP& Non-DP & Non-DP\\
attack & $ \epsilon=2 $ & $ \epsilon=2 $&$ \epsilon=2 $& $ \epsilon=4 $&$ \epsilon=4 $&$ \epsilon=4 $&$ \epsilon=8 $&$ \epsilon=8 $&$ \epsilon=8 $& $ \epsilon=\infty$& $ \epsilon=\infty$& $ \epsilon=\infty$& $ \epsilon=\infty$& $ \epsilon=\infty$  \\
magnitude&adv $\sfrac{2}{255}$& robust& accurate&adv $\sfrac{2}{255}$& robust& accurate&adv $\sfrac{2}{255}$& robust& accurate&adv $\sfrac{2}{255}$& robust& accurate& adv $\sfrac{8}{255}$& accurate \\

\hline $\gamma=0$ & 90.46\% & $89.69\%$ & $92.87\%$ & 91.12\% & $90.91\%$ & $93.41\%$ & 91.70\% & $91.22\%$ & $93.64\%$ & 93.42\% & $94.29\%$ & $94.55\%$ & $87.03\%$ & $95.25\%$ \\

$\gamma=\sfrac{2}{255}$ & 83.83\% & $81.05\%$ & $33.21\%$ & 85.28\% & $82.53\%$ & $57.80\%$ & 86.12\% & $83.02\%$ & $68.90\%$ & 89.07\% & $79.79\%$ & $59.56\%$ &-- & -- \\

$\gamma=\sfrac{4}{255}$ & 75.56\% & $68.85\%$ & $0.16\%$ & 77.73\% & $70.21\%$ & $9.69\%$ & 78.62\% & $71.08\%$ & $28.09\%$ & 83.07\% & $53.56\%$ & $15.99\%$ & -- & -- \\

$\gamma=\sfrac{8}{255}$ & 53.61\% & $39.63\%$ & $0.00\%$ & 56.90\% & $38.39\%$ & $0.00\%$ & 57.84\% & $39.28\%$ & $0.01\%$ & 66.99\% & $8.14\%$ & $0.00\%$ & $53.49\%$ & $0.00\%$ \\

$\gamma=\sfrac{16}{255}$ & 8.05\% & $1.20\%$ & $0.00\%$ & 10.30\% & $0.65\%$ & $0.00\%$ & 11.31\% & $0.91\%$ & $0.00\%$ & 20.67\% & $0.00\%$ & $0.00\%$ & $18.13\%$ & $0.00\%$ \\
\hline
\end{tabular}
}
\label{tab:cifar with madry}
\end{table}

\begin{table}[H]
    \centering
    \caption{Natural and robust accuracy of SimCLRv2 \citep{chen2020simple} and ResNet50 \citep{robustness} on CIFAR10 under 20 steps $l_2$ PGD attack. Here \textcolor{cyan}{cyan} columns are adversarial training and white columns are natural training. Adversarial training and \textit{robust} hyper-parameters are obtained by grid search over $\eta$ and $R$ against $l_2(0.25)$, and \textit{natural} hyper-parameters are adopted from \cite{tramer2020differentially}. 
    }
\resizebox{\linewidth}{!}{
\setlength\tabcolsep{1.2pt}
\begin{tabular}{c|acc|acc|acc|acc|ac}
\hline
&\multicolumn{12}{c}{SimCLRv2 pre-trained on unlabelled ImageNet}&\multicolumn{2}{|c}{ResNet50}
\\
\hline
 & DP& DP& DP& DP& DP& DP& DP & DP& DP& Non-DP& Non-DP & Non-DP& Non-DP & Non-DP\\
attack & $ \epsilon=2 $ & $ \epsilon=2 $&$ \epsilon=2 $& $ \epsilon=4 $&$ \epsilon=4 $&$ \epsilon=4 $&$ \epsilon=8 $&$ \epsilon=8 $&$ \epsilon=8 $& $ \epsilon=\infty$& $ \epsilon=\infty$& $ \epsilon=\infty$& $ \epsilon=\infty$& $ \epsilon=\infty$  \\
magnitude&adv 0.25& robust& accurate&adv 0.25& robust& accurate&adv 0.25& robust& accurate&adv 0.25& robust& accurate& adv $0.5$& accurate \\

\hline $\gamma=0$ & 91.07\% & $89.69\%$ & $92.87\%$ & 91.19\% & $90.91\%$ & $93.41\%$ & 91.87\% & $91.22\%$ & $93.64\%$ & 93.88\% & $94.29\%$ & $94.55\%$ & $90.83\%$ & $95.25\%$ \\

$\gamma=0.25$ & 82.69\% & $82.12\%$ & $59.91\%$ & 83.69\% & $83.35\%$ & $74.10\%$ & 84.27\% & $83.77\%$ & $79.03\%$ & 85.20\% & $82.91\%$ & $72.63\%$ & $82.34\%$ & $8.66\%$ \\

$\gamma=0.5$ & 70.54\% & $71.99\%$ & $12.76\%$ & 72.42\% & $72.79\%$ & $40.97\%$ & 72.00\% & $73.08\%$ & $54.53\%$ & 71.22\% & $63.32\%$ & $35.95\%$ & $70.17\%$ & $0.28\%$ \\

$\gamma=1.0$ & 38.57\% & 46.30\% & 9.49\% & 42.97\% & 44.46\% & 8.97\% & 40.17\% & 44.65\% & 9.68\% & 39.74\% & 33.34\% & 0.98\% & 40.47\% & 0.00\% \\

\hline
\end{tabular}
}

    \label{tab:cifar with madry l2}
\end{table}

We evaluate the robust and natural accuracy on the DP models in \citet{tramer2020differentially}, considering two groups of hyper-parameters: the \textit{natural} hyper-parameters reproduced from \citet{tramer2020differentially} that has highest natural accuracy, and the \textit{robust} hyper-parameters from a grid search on $(R,\eta)$ for the highest robust accuracy. From \Cref{tab:cifar with madry} and \Cref{tab:cifar with madry l2}, we see that even under the same privacy constraint (including the non-DP scenario), the robustness from different hyper-parameters can be fundamentally different. For example, DP SimCLR at $\epsilon=2$ can be either very robust ($\approx 70\%$ accuracy at $\gamma=4/255$) or not robust at all ($0.16\%$ accuracy). Consequently, our results may explain the mis-understanding of previous researches by the improper choice of the hyper-parameters.

Scrutinizing the natural training with robust hyper-parameters, we see that, across all $l_\infty$ attack magnitudes $\gamma=\{2/255,4/255,8/255,16/255\}$ and $l_2$ ones $\gamma=\{0,0.25,0.5,1.0\}$, DP SimCLR can be more robust than the non-DP SimCLR, in fact comparable to the adversarially trained ResNet50 that is benchmarked in \citet{robustness} and to the adversarially trained DP SimCLR. To be assured, we further demonstrate that our choice of small $R$ and large $\eta$ is consistently robust on Fashion MNIST, CIFAR10 and CelebA in \Cref{app:more tables}.

\subsection{Pareto optimality on accuracy and robustness}

\begin{figure}[!htb]
\vspace{-.5cm}
\centering
\includegraphics[width=0.55\linewidth]{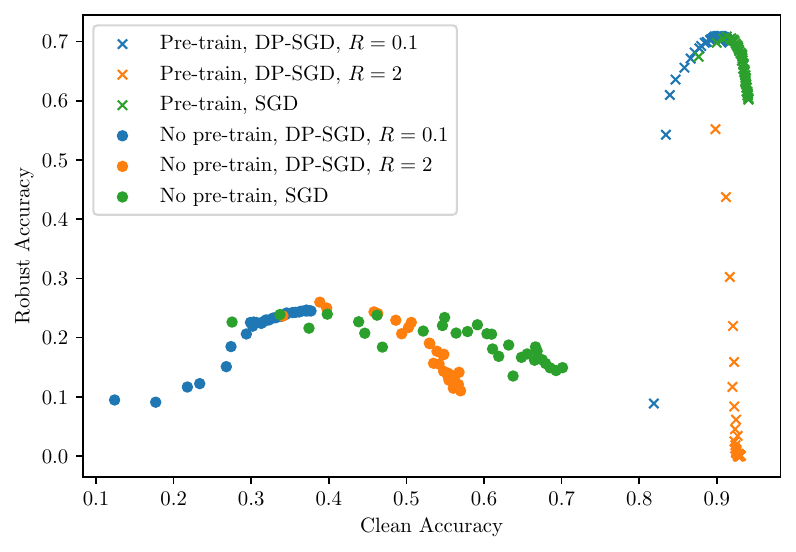}
    \vspace{-0.2cm}
    \caption{Robust and natural accuracy on CIFAR10 at different iterations. Dots are CNN from \citet{papernot2020tempered}. Crosses are SimCLR from \citet{tramer2020differentially}.
    See details in \Cref{app:hyperparameters}.
    }
    \label{fig:rob_acc_vs_clean_acc_cifar10_simclr}
\end{figure}

In the standard non-DP regime, the tradeoff between the accuracy and the robustness is well-known \citep{robustness}. We extend the Pareto statement in \Cref{thm:DP is robust than nat} to DP deep learning, thus adding the privacy dimension into the privacy-accuracy-robustness tradeoff. In \Cref{fig:rob_acc_vs_clean_acc_cifar10_simclr}, we show that two strong DP models on CIFAR10 (one pre-trained, the other not) achieve Pareto optimality with proper hyperparamters, and thus cannot be dominated by any natural classifiers. This can be observed by the fact that no green cross (or dot) is to the top right of all blue crosses (or dots), meaning that any natural classifier may have better robustness or higher accuracy, but not both. Therefore, our observation supports the claim that DP neural networks can be Pareto optimal in terms of the robustness and the accuracy.

\subsection{DP neural networks can be robust against general attacks}

Following the claim in \Cref{sec:small clip regime} that DP neural networks can be robust against $l_2$ and $l_\infty$ PGD attacks, we now demonstrate the transferability of DP neural networks' robustness against different attacks.

\begin{table}[!htb]
\centering
\caption{Natural and robust accuracy of FGSM\citep{goodfellow2014explaining}, BIM\citep{kurakin2018adversarial}, PGD$_{\infty} $\citep{madry2017towards}, APGD$_{\infty} $\citep{croce2020reliable}, PGD$_2 $\citep{madry2017towards} and APGD$_2 $\citep{croce2020reliable} on CIFAR10 under general adversarial attacks. Same model as \Cref{tab:cifar with madry} with the \textit{robust} hyper-parameters. See detailed attack settings in \Cref{app:hyperparameters}.
    }
\setlength\tabcolsep{3pt}
\begin{tabular}{cccccccc}
\hline & Natural & FGSM & BIM & PGD$_{\infty} $ & APGD$_{\infty} $ & PGD$_2 $ & APGD$_2 $ \\
\hline DP , $\epsilon=2$ & $89.86\%$ & $69.73\%$ &$68.85\%$ & $68.85\%$ & $68.85\%$ &$72.10\%$ & $71.97\%$ \\
DP , $\epsilon=4$ & $90.91\%$ & $71.13\%$ & $70.21\%$ & $70.21\%$ & $70.21\%$ & $72.79\%$ & $72.63\%$ \\
DP , $\epsilon=8$ & $91.22\%$ & $71.88\%$ & $71.08\%$ & $71.08\%$ & $71.06\%$  &$73.08\%$ & $72.99\%$\\
Non-DP & $94.29\%$ & $56.24\%$ & $53.56 \%$ & $53.56\%$ & $53.56\%$ & $63.32\%$ & $62.93\%$ \\
\hline
\end{tabular}
    \label{tab:multi-atk-cifar10_rob}
\end{table}

This is interesting in the sense that DP mechanism does not intentionally defend against any adversarial attack, while the adversarial training \citep{goodfellow2014explaining} usually specifically targets a particular attack, e.g. PGD attack is defensed by PGD adversarial training. In \Cref{tab:multi-atk-cifar10_rob}, we attack on the robust models from \Cref{tab:cifar with madry}, with the \textit{robust} hyper-parameters. We consistently observe that DP models can be adversarially robust and more so than the non-DP ones on MNIST/Fashion MNIST/CelebA in \Cref{app:more tables}, if the hyper-parameters $(R,\eta)$ are set properly.

\subsection{Large scale experiments on CelebA face datasets}

We further validate our claims on CelebA \citep{liu2015faceattributes}, a public high-resolution ($178\times 218 $ pixels) image dataset, consisting of over 200,000 real human faces that are supposed to be protected against privacy risks.  We train ResNet18 (\citet{he2016deep}, 11 million parameters) and Vision Transformer (ViT, \citet{dosovitskiy2020image}, 6 million parameters) with DP-RMSprop. Both models are implemented by \citet{rw2019timm} and pretrained on ImageNet. The experiment can be reproduced using the DP vision codebase \texttt{Private Vision} by \citep{bu2022scalable}.

\begin{table}[!htb]
    \centering
    \caption{Natural and robust accuracy on CelebA with label `Smiling', DP-RMSprop, under 20 steps $l_\infty$ PGD attack. 
    Here the hyper-parameters have \textit{not} been carefully searched for the best robustness. 
    See details in \Cref{app:hyperparameters}.
    }
    \setlength\tabcolsep{2.3pt}
\begin{tabular}{c|cccc|cccc}
\hline&\multicolumn{4}{c|}{ResNet18}&\multicolumn{4}{c}{ViT}\\
\hline
attack&  DP& DP& DP& Non-DP & DP& DP& DP&Non-DP \\
magnitude & $ \epsilon=2$&$ \epsilon=4 $ &$ \epsilon=8 $ & $ \epsilon=\infty $&$ \epsilon=2 $&$ \epsilon=4 $&$ \epsilon=8 $& $ \epsilon=\infty$ \\
\hline 
$ \gamma=0 $ &               $80.10\%$ & $85.10\%$ &$88.48\%$ &$91.91\%$  & $92.30\%$ & $92.33\%$ & $92.09\%$& $92.87\%$ \\

$ \gamma=\sfrac{2}{255} $  & $1.26\%$ & $0.47\%$ & $1.03\%$ & $1.19\%$ & $1.42\%$ & $2.02\%$ & $10.35\%$& $0.08\%$ \\

$ \gamma=\sfrac{4}{255} $  & $0.01\%$ & $0.01\%$ & $0.00\%$& $0.00\%$ & $0.00\%$ & $0.00\%$ & $0.00\%$ & $0.00\%$ \\

$ \gamma=\sfrac{8}{255} $ & $0.00\%$ & $0.00\%$ & $0.00\%$ & $0.00\%$ & $0.00\%$ & $0.00\%$ & $0.00\%$ & $0.00\%$ \\

$ \gamma=\sfrac{16}{255} $ & $0.00\%$ & $0.00\%$ & $0.00\%$ & $0.00\%$ & $0.00\%$ & $0.00\%$ & $0.00\%$ & $0.00\%$ \\
\hline
\end{tabular}
    \label{tab:resnet18 and vit}
\end{table}

\begin{table}[!htb]
    \centering
    \caption{Natural and robust accuracy of FGSM\citep{goodfellow2014explaining}, BIM\citep{kurakin2018adversarial}, PGD$_{\infty} $\citep{madry2017towards}, APGD$_{\infty} $\citep{croce2020reliable}, PGD$_2 $\citep{madry2017towards} and APGD$_2 $\citep{croce2020reliable} on CelebA with label `Smiling' under general adversarial attacks. Same ResNet18 as \Cref{tab:resnet18 and vit}. See detailed attack settings in \Cref{app:hyperparameters}.}
\setlength\tabcolsep{3pt}
\begin{tabular}{cccccccc}
\hline & Natural & FGSM & BIM& PGD$_{\infty} $ & APGD$_{\infty} $ & PGD$_2 $ & APGD$_2 $ \\
\hline 
DP , $\epsilon=2$ & $80.10\%$ & $24.47\%$ &$1.24\%$ & $1.26\%$ & $1.18\%$ &$47.02\%$ & $46.25\%$ \\
DP , $\epsilon=4$ & $85.10\%$ & $24.40\%$ & $0.45\%$ & $0.47\%$ & $0.41\%$ & $56.92\%$ & $56.09\%$ \\
DP , $\epsilon=8$ & $88.48\%$ & $29.32\%$ & $0.97\%$ & $1.03\%$ & $0.41\%$  &$57.40\%$ & $56.69\%$\\
Non-DP & $91.91\%$ & $22.94\%$ & $1.09 \%$ & $1.19\%$ & $0.68\%$ & $66.89\%$ & $66.13\%$ \\
\hline
\end{tabular}    
\label{tab:resnet18 general attacks}
\end{table}

\vspace{1cm}
In \Cref{tab:resnet18 and vit} and \Cref{tab:resnet18 general attacks}, we observe that DP ResNet18 and ViT are almost as adversarially robust as their non-DP counterparts, if not more robust. These observations are consistent with those of simpler models on tiny images (c.f. \Cref{tab:cifar with madry} and \Cref{tab:multi-atk-cifar10_rob}). We note that unlike the linear classifiers on CIFAR10 in \Cref{tab:cifar with madry}, training all layers on CelebA in \Cref{tab:resnet18 and vit} are much more vulnerable even at $\gamma=2/255$.

\section{Discussion}


Through the lens of theoretical analysis and extensive experiments, we have shown that differentially private models can be adversarially robust and sometimes even more robust than the naturally trained models. Moreover, DP models can be Pareto optimal in the sense that a more accurate natural model must be less robust (see \Cref{thm:DP is robust than nat} and \Cref{fig:rob_acc_vs_clean_acc_cifar10_simclr}). Our conclusion holds for various attacks with different magnitudes, from linear models to large vision models, from grey-scale images to real face datasets, and from SGD to adaptive optimizers. We not only are the first to reveal this possibility of achieving privacy and robustness simultaneously, but also are the first to offer practical guidelines for such important goal (see \Cref{sec:experiments}). To be concrete, we demonstrate that hyper-parameters -- clipping norm $R$ and learning rate $\eta$ -- exert a lot of influence on the robustness and accuracy, while remaining equally private. We hope that our insights will encourage the practitioners to adopt techniques that protect the privacy and robustness in real-world applications.

For future directions, a more thorough study of private and robustness learning is desirable, by extending to language models, recommendation systems, and so on. We believe a new analysis when all parameters are trainable will be challenging but enlightening. Especially, given that larger models are empirically more accurate under the fixed privacy budget, it would be interesting to understand whether the robustness also improves, or at least persists, with larger model sizes. Another direction is to further investigate the adversarial training with DP optimizers, whose performance may go beyond our Pareto frontier (of the robustness and the accuracy) that is based on the natural training.

\bibliography{main}

\begin{thebibliography}{47}
\providecommand{\natexlab}[1]{#1}
\providecommand{\url}[1]{\texttt{#1}}
\expandafter\ifx\csname urlstyle\endcsname\relax
  \providecommand{\doi}[1]{doi: #1}\else
  \providecommand{\doi}{doi: \begingroup \urlstyle{rm}\Url}\fi

\bibitem[Abadi et~al.(2016)Abadi, Chu, Goodfellow, McMahan, Mironov, Talwar,
  and Zhang]{abadi2016deep}
Martin Abadi, Andy Chu, Ian Goodfellow, H~Brendan McMahan, Ilya Mironov, Kunal
  Talwar, and Li~Zhang.
\newblock Deep learning with differential privacy.
\newblock In \emph{Proceedings of the 2016 ACM SIGSAC conference on computer
  and communications security}, pp.\  308--318, 2016.

\bibitem[Boenisch et~al.(2021)Boenisch, Sperl, and
  B{\"o}ttinger]{boenisch2021gradient}
Franziska Boenisch, Philip Sperl, and Konstantin B{\"o}ttinger.
\newblock Gradient masking and the underestimated robustness threats of
  differential privacy in deep learning.
\newblock \emph{arXiv preprint arXiv:2105.07985}, 2021.

\bibitem[Bu et~al.(2020)Bu, Dong, Long, and Su]{bu2020deep}
Zhiqi Bu, Jinshuo Dong, Qi~Long, and Weijie~J Su.
\newblock Deep learning with gaussian differential privacy.
\newblock \emph{Harvard data science review}, 2020\penalty0 (23), 2020.

\bibitem[Bu et~al.(2022{\natexlab{a}})Bu, Mao, and Xu]{bu2022scalable}
Zhiqi Bu, Jialin Mao, and Shiyun Xu.
\newblock Scalable and efficient training of large convolutional neural
  networks with differential privacy.
\newblock \emph{Advances in Neural Information Processing Systems},
  35:\penalty0 38305--38318, 2022{\natexlab{a}}.

\bibitem[Bu et~al.(2022{\natexlab{b}})Bu, Wang, Zha, and
  Karypis]{bu2022differentially}
Zhiqi Bu, Yu-Xiang Wang, Sheng Zha, and George Karypis.
\newblock Differentially private bias-term only fine-tuning of foundation
  models.
\newblock In \emph{Workshop on Trustworthy and Socially Responsible Machine
  Learning, NeurIPS 2022}, 2022{\natexlab{b}}.

\bibitem[Bu et~al.(2023{\natexlab{a}})Bu, Wang, Dai, and
  Long]{bu2021convergence}
Zhiqi Bu, Hua Wang, Zongyu Dai, and Qi~Long.
\newblock On the convergence and calibration of deep learning with differential
  privacy.
\newblock \emph{Transactions on Machine Learning Research}, 2023{\natexlab{a}}.

\bibitem[Bu et~al.(2023{\natexlab{b}})Bu, Wang, Zha, and
  Karypis]{bu2022automatic}
Zhiqi Bu, Yu-Xiang Wang, Sheng Zha, and George Karypis.
\newblock Automatic clipping: Differentially private deep learning made easier
  and stronger.
\newblock \emph{Advances in Neural Information Processing Systems},
  2023{\natexlab{b}}.

\bibitem[Carlini et~al.(2021)Carlini, Tramer, Wallace, Jagielski, Herbert-Voss,
  Lee, Roberts, Brown, Song, Erlingsson, et~al.]{carlini2021extracting}
Nicholas Carlini, Florian Tramer, Eric Wallace, Matthew Jagielski, Ariel
  Herbert-Voss, Katherine Lee, Adam Roberts, Tom Brown, Dawn Song, Ulfar
  Erlingsson, et~al.
\newblock Extracting training data from large language models.
\newblock In \emph{30th USENIX Security Symposium (USENIX Security 21)}, pp.\
  2633--2650, 2021.

\bibitem[Chen et~al.(2020)Chen, Kornblith, Norouzi, and Hinton]{chen2020simple}
Ting Chen, Simon Kornblith, Mohammad Norouzi, and Geoffrey Hinton.
\newblock A simple framework for contrastive learning of visual
  representations.
\newblock In \emph{International conference on machine learning}, pp.\
  1597--1607. PMLR, 2020.

\bibitem[Croce \& Hein(2020)Croce and Hein]{croce2020reliable}
Francesco Croce and Matthias Hein.
\newblock Reliable evaluation of adversarial robustness with an ensemble of
  diverse parameter-free attacks.
\newblock In \emph{International conference on machine learning}, pp.\
  2206--2216. PMLR, 2020.

\bibitem[De et~al.(2022)De, Berrada, Hayes, Smith, and Balle]{de2022unlocking}
Soham De, Leonard Berrada, Jamie Hayes, Samuel~L Smith, and Borja Balle.
\newblock Unlocking high-accuracy differentially private image classification
  through scale.
\newblock \emph{arXiv preprint arXiv:2204.13650}, 2022.

\bibitem[Deng et~al.(2009)Deng, Dong, Socher, Li, Li, and
  Fei-Fei]{deng2009imagenet}
Jia Deng, Wei Dong, Richard Socher, Li-Jia Li, Kai Li, and Li~Fei-Fei.
\newblock Imagenet: A large-scale hierarchical image database.
\newblock In \emph{2009 IEEE Conference on Computer Vision and Pattern
  Recognition}, pp.\  248--255, 2009.
\newblock \doi{10.1109/CVPR.2009.5206848}.

\bibitem[Dosovitskiy et~al.(2020)Dosovitskiy, Beyer, Kolesnikov, Weissenborn,
  Zhai, Unterthiner, Dehghani, Minderer, Heigold, Gelly,
  et~al.]{dosovitskiy2020image}
Alexey Dosovitskiy, Lucas Beyer, Alexander Kolesnikov, Dirk Weissenborn,
  Xiaohua Zhai, Thomas Unterthiner, Mostafa Dehghani, Matthias Minderer, Georg
  Heigold, Sylvain Gelly, et~al.
\newblock An image is worth 16x16 words: Transformers for image recognition at
  scale.
\newblock \emph{arXiv preprint arXiv:2010.11929}, 2020.

\bibitem[Dwork et~al.(2014)Dwork, Roth, et~al.]{dwork2014algorithmic}
Cynthia Dwork, Aaron Roth, et~al.
\newblock The algorithmic foundations of differential privacy.
\newblock \emph{Found. Trends Theor. Comput. Sci.}, 9\penalty0 (3-4):\penalty0
  211--407, 2014.

\bibitem[Engstrom et~al.(2019)Engstrom, Ilyas, Salman, Santurkar, and
  Tsipras]{robustness}
Logan Engstrom, Andrew Ilyas, Hadi Salman, Shibani Santurkar, and Dimitris
  Tsipras.
\newblock Robustness (python library), 2019.
\newblock URL \url{https://github.com/MadryLab/robustness}.

\bibitem[Fowl et~al.(2021)Fowl, Goldblum, Chiang, Geiping, Czaja, and
  Goldstein]{fowl2021adversarial}
Liam Fowl, Micah Goldblum, Ping-yeh Chiang, Jonas Geiping, Wojciech Czaja, and
  Tom Goldstein.
\newblock Adversarial examples make strong poisons.
\newblock \emph{Advances in Neural Information Processing Systems},
  34:\penalty0 30339--30351, 2021.

\bibitem[Goodfellow et~al.(2014)Goodfellow, Shlens, and
  Szegedy]{goodfellow2014explaining}
Ian~J Goodfellow, Jonathon Shlens, and Christian Szegedy.
\newblock Explaining and harnessing adversarial examples.
\newblock \emph{arXiv preprint arXiv:1412.6572}, 2014.

\bibitem[He et~al.(2016)He, Zhang, Ren, and Sun]{he2016deep}
Kaiming He, Xiangyu Zhang, Shaoqing Ren, and Jian Sun.
\newblock Deep residual learning for image recognition.
\newblock In \emph{Proceedings of the IEEE conference on computer vision and
  pattern recognition}, pp.\  770--778, 2016.

\bibitem[Hong et~al.(2020)Hong, Chandrasekaran, Kaya, Dumitra{\c{s}}, and
  Papernot]{hong2020effectiveness}
Sanghyun Hong, Varun Chandrasekaran, Yi{\u{g}}itcan Kaya, Tudor Dumitra{\c{s}},
  and Nicolas Papernot.
\newblock On the effectiveness of mitigating data poisoning attacks with
  gradient shaping.
\newblock \emph{arXiv preprint arXiv:2002.11497}, 2020.

\bibitem[Huang et~al.(2008)Huang, Mattar, Berg, and
  Learned-Miller]{huang2008labeled}
Gary~B Huang, Marwan Mattar, Tamara Berg, and Eric Learned-Miller.
\newblock Labeled faces in the wild: A database forstudying face recognition in
  unconstrained environments.
\newblock In \emph{Workshop on faces in'Real-Life'Images: detection, alignment,
  and recognition}, 2008.

\bibitem[Klause et~al.(2022)Klause, Ziller, Rueckert, Hammernik, and
  Kaissis]{klause2022differentially}
Helena Klause, Alexander Ziller, Daniel Rueckert, Kerstin Hammernik, and
  Georgios Kaissis.
\newblock Differentially private training of residual networks with scale
  normalisation.
\newblock \emph{arXiv preprint arXiv:2203.00324}, 2022.

\bibitem[Krizhevsky et~al.(2009)Krizhevsky, Hinton,
  et~al.]{krizhevsky2009learning}
Alex Krizhevsky, Geoffrey Hinton, et~al.
\newblock Learning multiple layers of features from tiny images.
\newblock 2009.

\bibitem[Kurakin et~al.(2018)Kurakin, Goodfellow, and
  Bengio]{kurakin2018adversarial}
Alexey Kurakin, Ian~J Goodfellow, and Samy Bengio.
\newblock Adversarial examples in the physical world.
\newblock In \emph{Artificial intelligence safety and security}, pp.\  99--112.
  Chapman and Hall/CRC, 2018.

\bibitem[Kurakin et~al.(2022)Kurakin, Chien, Song, Geambasu, Terzis, and
  Thakurta]{kurakin2022toward}
Alexey Kurakin, Steve Chien, Shuang Song, Roxana Geambasu, Andreas Terzis, and
  Abhradeep Thakurta.
\newblock Toward training at imagenet scale with differential privacy.
\newblock \emph{arXiv preprint arXiv:2201.12328}, 2022.

\bibitem[LeCun et~al.(1998)LeCun, Bottou, Bengio, and
  Haffner]{lecun1998gradient}
Yann LeCun, L{\'e}on Bottou, Yoshua Bengio, and Patrick Haffner.
\newblock Gradient-based learning applied to document recognition.
\newblock \emph{Proceedings of the IEEE}, 86\penalty0 (11):\penalty0
  2278--2324, 1998.
\newblock \doi{10.1109/5.726791}.

\bibitem[Lecuyer et~al.(2019)Lecuyer, Atlidakis, Geambasu, Hsu, and
  Jana]{lecuyer2019certified}
Mathias Lecuyer, Vaggelis Atlidakis, Roxana Geambasu, Daniel Hsu, and Suman
  Jana.
\newblock Certified robustness to adversarial examples with differential
  privacy.
\newblock In \emph{2019 IEEE Symposium on Security and Privacy (SP)}, pp.\
  656--672. IEEE, 2019.

\bibitem[Li et~al.(2021)Li, Tramer, Liang, and Hashimoto]{li2021large}
Xuechen Li, Florian Tramer, Percy Liang, and Tatsunori Hashimoto.
\newblock Large language models can be strong differentially private learners.
\newblock \emph{arXiv preprint arXiv:2110.05679}, 2021.

\bibitem[Liu et~al.(2015)Liu, Luo, Wang, and Tang]{liu2015faceattributes}
Ziwei Liu, Ping Luo, Xiaogang Wang, and Xiaoou Tang.
\newblock Deep learning face attributes in the wild.
\newblock In \emph{Proceedings of International Conference on Computer Vision
  (ICCV)}, December 2015.

\bibitem[Madry et~al.(2017)Madry, Makelov, Schmidt, Tsipras, and
  Vladu]{madry2017towards}
Aleksander Madry, Aleksandar Makelov, Ludwig Schmidt, Dimitris Tsipras, and
  Adrian Vladu.
\newblock Towards deep learning models resistant to adversarial attacks.
\newblock \emph{arXiv preprint arXiv:1706.06083}, 2017.

\bibitem[McMahan et~al.(2017)McMahan, Ramage, Talwar, and
  Zhang]{mcmahan2017learning}
H~Brendan McMahan, Daniel Ramage, Kunal Talwar, and Li~Zhang.
\newblock Learning differentially private recurrent language models.
\newblock \emph{arXiv preprint arXiv:1710.06963}, 2017.

\bibitem[Mehta et~al.(2022)Mehta, Thakurta, Kurakin, and
  Cutkosky]{mehta2022large}
Harsh Mehta, Abhradeep Thakurta, Alexey Kurakin, and Ashok Cutkosky.
\newblock Large scale transfer learning for differentially private image
  classification.
\newblock \emph{arXiv preprint arXiv:2205.02973}, 2022.

\bibitem[Mejia et~al.(2019)Mejia, Gamble, Hampel-Arias, Lomnitz, Lopatina,
  Tindall, and Barrios]{mejia2019robust}
Felipe~A Mejia, Paul Gamble, Zigfried Hampel-Arias, Michael Lomnitz, Nina
  Lopatina, Lucas Tindall, and Maria~Alejandra Barrios.
\newblock Robust or private? adversarial training makes models more vulnerable
  to privacy attacks.
\newblock \emph{arXiv preprint arXiv:1906.06449}, 2019.

\bibitem[Netzer et~al.(2011)Netzer, Wang, Coates, Bissacco, Wu, and
  Ng]{netzer2011reading}
Yuval Netzer, Tao Wang, Adam Coates, Alessandro Bissacco, Bo~Wu, and Andrew~Y
  Ng.
\newblock Reading digits in natural images with unsupervised feature learning.
\newblock 2011.

\bibitem[Nori et~al.(2021)Nori, Caruana, Bu, Shen, and
  Kulkarni]{nori2021accuracy}
Harsha Nori, Rich Caruana, Zhiqi Bu, Judy~Hanwen Shen, and Janardhan Kulkarni.
\newblock Accuracy, interpretability, and differential privacy via explainable
  boosting.
\newblock In \emph{International Conference on Machine Learning}, pp.\
  8227--8237. PMLR, 2021.

\bibitem[Papernot et~al.(2020)Papernot, Thakurta, Song, Chien, and
  Erlingsson]{papernot2020tempered}
Nicolas Papernot, Abhradeep Thakurta, Shuang Song, Steve Chien, and Ulfar
  Erlingsson.
\newblock Tempered sigmoid activations for deep learning with differential
  privacy.
\newblock \emph{arXiv preprint arXiv:2007.14191}, pp.\ ~10, 2020.

\bibitem[Radford et~al.(2019)Radford, Wu, Child, Luan, Amodei, Sutskever,
  et~al.]{radford2019language}
Alec Radford, Jeffrey Wu, Rewon Child, David Luan, Dario Amodei, Ilya
  Sutskever, et~al.
\newblock Language models are unsupervised multitask learners.
\newblock \emph{OpenAI blog}, 1\penalty0 (8):\penalty0 9, 2019.

\bibitem[Song et~al.(2019)Song, Shokri, and Mittal]{song2019privacy}
Liwei Song, Reza Shokri, and Prateek Mittal.
\newblock Privacy risks of securing machine learning models against adversarial
  examples.
\newblock In \emph{Proceedings of the 2019 ACM SIGSAC Conference on Computer
  and Communications Security}, pp.\  241--257, 2019.

\bibitem[Song et~al.(2021)Song, Steinke, Thakkar, and
  Thakurta]{song2021evading}
Shuang Song, Thomas Steinke, Om~Thakkar, and Abhradeep Thakurta.
\newblock Evading the curse of dimensionality in unconstrained private glms.
\newblock In \emph{International Conference on Artificial Intelligence and
  Statistics}, pp.\  2638--2646. PMLR, 2021.

\bibitem[Tieleman et~al.(2012)Tieleman, Hinton, et~al.]{tieleman2012lecture}
Tijmen Tieleman, Geoffrey Hinton, et~al.
\newblock Lecture 6.5-rmsprop: Divide the gradient by a running average of its
  recent magnitude.
\newblock \emph{COURSERA: Neural Networks for Machine Learning}, 4\penalty0
  (2):\penalty0 26--31, 2012.

\bibitem[Tramer \& Boneh(2020)Tramer and Boneh]{tramer2020differentially}
Florian Tramer and Dan Boneh.
\newblock Differentially private learning needs better features (or much more
  data).
\newblock \emph{arXiv preprint arXiv:2011.11660}, 2020.

\bibitem[Tursynbek et~al.(2020)Tursynbek, Petiushko, and
  Oseledets]{tursynbek2020robustness}
Nurislam Tursynbek, Aleksandr Petiushko, and Ivan Oseledets.
\newblock Robustness threats of differential privacy.
\newblock \emph{arXiv preprint arXiv:2012.07828}, 2020.

\bibitem[Wightman(2019)]{rw2019timm}
Ross Wightman.
\newblock Pytorch image models.
\newblock \url{https://github.com/rwightman/pytorch-image-models}, 2019.

\bibitem[Xiao et~al.(2017)Xiao, Rasul, and Vollgraf]{xiao2017fashion}
Han Xiao, Kashif Rasul, and Roland Vollgraf.
\newblock Fashion-mnist: a novel image dataset for benchmarking machine
  learning algorithms, 2017.

\bibitem[Xu et~al.(2021)Xu, Liu, Li, Jain, and Tang]{xu2021robust}
Han Xu, Xiaorui Liu, Yaxin Li, Anil Jain, and Jiliang Tang.
\newblock To be robust or to be fair: Towards fairness in adversarial training.
\newblock In \emph{International Conference on Machine Learning}, pp.\
  11492--11501. PMLR, 2021.

\bibitem[Yang et~al.(2022)Yang, Liu, and Mirzasoleiman]{yang2022not}
Yu~Yang, Tian~Yu Liu, and Baharan Mirzasoleiman.
\newblock Not all poisons are created equal: Robust training against data
  poisoning.
\newblock In \emph{International Conference on Machine Learning}, pp.\
  25154--25165. PMLR, 2022.

\bibitem[Zaken et~al.(2022)Zaken, Goldberg, and Ravfogel]{zaken2021bitfit}
Elad~Ben Zaken, Yoav Goldberg, and Shauli Ravfogel.
\newblock Bitfit: Simple parameter-efficient fine-tuning for transformer-based
  masked language-models.
\newblock In \emph{Proceedings of the 60th Annual Meeting of the Association
  for Computational Linguistics (Volume 2: Short Papers)}, pp.\  1--9, 2022.

\bibitem[Zhu et~al.(2019)Zhu, Liu, and Han]{zhu2019deep}
Ligeng Zhu, Zhijian Liu, and Song Han.
\newblock Deep leakage from gradients.
\newblock \emph{Advances in Neural Information Processing Systems}, 32, 2019.

\end{thebibliography}
\bibliographystyle{tmlr}

\clearpage
\appendix

\section{Proofs}
\label{app:proofs}
\begin{proof}[Proof of \Cref{thm:robust classifier}] [Extended from \citet{xu2021robust}] By \citet[Lemma 2]{xu2021robust}, according to the data symmetry in \eqref{eq:x distribution}, the optimal linear classifier has the form 
$$
1,\cdots,1,b_{\gamma}=\underset{\w,b}{\arg\min} \mathcal{R}_{\gamma}(f(\cdot;\w,b)).
$$
Recall that \eqref{eq:robust error of any intercept} proves that for such linear classifier, the robust error is
\begin{align*}
\mathcal{R}_{\gamma}(f)=\frac{1}{2}\Phi\left(-\frac{\sqrt{d}(\theta-\gamma)}{\sigma}+\frac{1}{\sqrt{d} \sigma} \cdot b\right) 
+\frac{1}{2}\Phi\left(-\frac{\sqrt{d}(\theta-\gamma)}{K \sigma}-\frac{1}{K \sqrt{d} \sigma} \cdot b\right).
\end{align*}
where $\Phi$ is the cumulative distribution function of standard normal.

The optimal $b_\gamma$ to minimize $\mathcal{R}_{\gamma}(f)$ is achieved at the point that $\frac{\partial\mathcal{R}_{\gamma}(f)}{\partial b}=0$. Thus, $b_{\gamma}$ satisfies:
$$\phi\left(-\frac{\sqrt{d}(\theta-\gamma)}{\sigma}+\frac{b_\gamma}{\sqrt{d} \sigma} \right) \cdot \frac{1}{\sqrt{d}\sigma}
-\phi\left(-\frac{\sqrt{d}(\theta-\gamma)}{K \sigma}-\frac{b_\gamma}{K \sqrt{d} \sigma} \right)\cdot \frac{1}{K \sqrt{d} \sigma}=0$$
where $\phi$ is the probability density function of standard normal. This equals to
$$\phi\left(-\frac{\sqrt{d}(\theta-\gamma)}{\sigma}+\frac{b_\gamma}{\sqrt{d} \sigma} \right)
=\phi\left(-\frac{\sqrt{d}(\theta-\gamma)}{K \sigma}-\frac{b_\gamma}{K \sqrt{d} \sigma} \right)/K$$
and
\begin{align*}
K&=\phi\left(-\frac{\sqrt{d}(\theta-\gamma)}{K \sigma}-\frac{b_\gamma}{K \sqrt{d} \sigma} \right)/\phi\left(-\frac{\sqrt{d}(\theta-\gamma)}{\sigma}+\frac{b_\gamma}{\sqrt{d} \sigma} \right)
\\
&=e^{-\frac{1}{2}\left[\left(-\frac{\sqrt{d}(\theta-\gamma)}{K \sigma}-\frac{b_\gamma}{K \sqrt{d} \sigma} \right)^2-\left(-\frac{\sqrt{d}(\theta-\gamma)}{\sigma}+\frac{b_\gamma}{\sqrt{d} \sigma} \right)^2\right]}
\end{align*}
It is not hard to see
$$-2\log K=\left(-\frac{\sqrt{d}(\theta-\gamma)}{K \sigma}-\frac{b_\gamma}{K \sqrt{d} \sigma} \right)^2-\left(-\frac{\sqrt{d}(\theta-\gamma)}{\sigma}+\frac{b_\gamma}{\sqrt{d} \sigma} \right)^2$$
which re-arranges to a quadratic equation
$$b_\gamma^2 \frac{1}{ d\sigma^2}(1-\frac{1}{K^2})-b_\gamma \frac{2(\theta-\gamma)}{\sigma^2}(1+\frac{1}{K^2})+\frac{d(\theta-\gamma)^2}{\sigma^2}(1-\frac{1}{K^2})=2\log K.$$
The solution is therefore explicit as

$$
b_{\gamma}=\frac{K^{2}+1}{K^{2}-1} d (\theta-\gamma)-K \sqrt{\frac{4d^{2} (\theta-\gamma)^{2}}{\left(K^{2}-1\right)^{2}}+d \sigma^{2} q(K)},
$$
where $ q(K)=\frac{2 \log K}{K^{2}-1} $ which is a positive constant and only depends on $K$. By incorporating $b_{\gamma}$ into \eqref{eq:robust error of any intercept}, we can get the optimal robust error $\mathcal{R}_{\gamma}\left(f_{\gamma}\right)$:
\begin{align*}
\mathcal{R}_{\gamma}\left(f_{\gamma}\right)=\frac{1}{2}\Phi\left(B(K,\gamma)-K \sqrt{B(K,\gamma)^{2}+q(K)}\right) +\frac{1}{2}\Phi\left(-K B(K,\gamma)+\sqrt{B(K,\gamma)^{2}+q(K)}\right),
\end{align*}
where $B(K,\gamma)=\frac{2}{K^{2}-1} \frac{\sqrt{d} (\theta-\gamma)}{\sigma} $.
\end{proof}

\begin{proof}[Proof of \Cref{thm:DP is robust than nat}]
We denote the two roots of $\frac{\partial\mathcal{R}_\gamma(f(b))}{\partial b}=0$ as $b_\gamma^+$ and $b_\gamma^-$. Here $b_\gamma\equiv b_\gamma^-$. Clearly $\mathcal{R}_\gamma(b)$ is increasing in $(b_\gamma^-,b_\gamma^+)$. We hope to show $b_0\in(b_\gamma^-,b_\gamma^+) \forall \gamma>0$, so that $\mathcal{R}_\gamma(b)$ is also increasing in $(b_\gamma^-,b_0)$.

Note their Equation (17)
\begin{align*}
\mathcal{R}_\gamma(b)=\frac{1}{2}\Phi(-\frac{\sqrt{d}(\theta-\gamma)}{\sigma}+\frac{1}{\sqrt{d}\sigma}b)+\frac{1}{2}\Phi(-\frac{\sqrt{d}(\theta-\gamma)}{K\sigma}-\frac{1}{K\sqrt{d}\sigma}b)
\end{align*}
Taking derivative w.r.t. $b$
\begin{align*}
\frac{\partial\mathcal{R}_\gamma(b)}{\partial b}=\frac{1}{2\sqrt{d}\sigma}\phi(-\frac{\sqrt{d}(\theta-\gamma)}{\sigma}+\frac{1}{\sqrt{d}\sigma}b)-\frac{1}{2K\sqrt{d}\sigma}\phi(-\frac{\sqrt{d}(\theta-\gamma)}{K\sigma}-\frac{1}{K\sqrt{d}\sigma}b)
\end{align*}
Setting this derivative to 0:
\begin{align*}
0=K\phi(-\frac{\sqrt{d}(\theta-\gamma)}{\sigma}+\frac{1}{\sqrt{d}\sigma}b)-\phi(-\frac{\sqrt{d}(\theta-\gamma)}{K\sigma}-\frac{1}{K\sqrt{d}\sigma}b)
\end{align*}
which means
\begin{align*}
\frac{\phi(-\frac{\sqrt{d}(\theta-\gamma)}{K\sigma}-\frac{1}{K\sqrt{d}\sigma}b)}{\phi(-\frac{\sqrt{d}(\theta-\gamma)}{\sigma}+\frac{1}{\sqrt{d}\sigma}b)}=K
\end{align*}
Using the standard normal density $\phi(u)=e^{-u^2/2}$ and $\frac{\phi(u)}{\phi(v)}=e^{(v^2-u^2)/2}$, we have
\begin{align*}
&(-\frac{\sqrt{d}(\theta-\gamma)}{\sigma}+\frac{1}{\sqrt{d}\sigma}b)^2-(-\frac{\sqrt{d}(\theta-\gamma)}{K\sigma}-\frac{1}{K\sqrt{d}\sigma}b)^2=2\log K
\\
\Longrightarrow& K^2(-d(\theta-\gamma)+b)^2-(-d(\theta-\gamma)-b)^2=2d\sigma^2 K^2\log K
\\
\Longrightarrow& (K^2-1)b^2-2d(\theta-\gamma)(K^2+1)b+d^2(\theta-\gamma)^2(K^2-1)-2d\sigma^2 K^2\log K=0
\end{align*}
By $x=-\frac{b}{2a}\pm \frac{\sqrt{b^2-4ac}}{2a}=-\frac{b}{2a}\pm \sqrt{(\frac{b}{2a})^2-\frac{c}{a}}$, we know 
\begin{align*}
b_\gamma^\pm
&=\frac{K^{2}+1}{K^{2}-1} d (\theta-\gamma)\pm \sqrt{(\frac{K^{2}+1}{K^{2}-1} d (\theta-\gamma))^2-d^2(\theta-\gamma)^2+K^2 d\sigma^2 q(K)}
\\
&=\frac{K^{2}+1}{K^{2}-1} d (\theta-\gamma)\pm K \sqrt{\frac{4d^{2} (\theta-\gamma)^{2}}{\left(K^{2}-1\right)^{2}}+d \sigma^{2} q(K)}
\end{align*}

We now derive the sufficient condition that $b_0<b_{\gamma}^+$:
$$  \frac{K^{2}+1}{K^{2}-1} d (\theta)-K \sqrt{\frac{4d^{2} (\theta)^{2}}{\left(K^{2}-1\right)^{2}}+d \sigma^{2} q(K)}<\frac{K^{2}+1}{K^{2}-1} d (\theta-\gamma)+K \sqrt{\frac{4d^{2} (\theta-\gamma)^{2}}{\left(K^{2}-1\right)^{2}}+d \sigma^{2} q(K)}.$$
This is equivalent to
$$\frac{K^{2}+1}{K^{2}-1} d \gamma<K\left( \sqrt{\frac{4d^{2} (\theta-\gamma)^{2}}{\left(K^{2}-1\right)^{2}}+d \sigma^{2} q(K)}+ \sqrt{\frac{4d^{2} \theta^{2}}{\left(K^{2}-1\right)^{2}}+d \sigma^{2} q(K)}\right).$$
Therefore, it suffices to have
$$\frac{K^2+1}{2K}\gamma<|\theta-\gamma|+|\theta|$$

Finally, it is easy to see the Pareto statement $\mathcal{R}_0(f)<\mathcal{R}_0(f_\text{DP})\longrightarrow\mathcal{R}_\gamma(f)>\mathcal{R}_\gamma(f_\text{DP})$. A necessary but not sufficient condition for $\mathcal{R}_0(f)<\mathcal{R}_0(f_\text{DP})$ given that $b_0>b_\text{DP}$ is $b>b_\text{DP}$, since $b_0$ is a minimizer which means $\mathcal{R}_0$ is decreasing on the interval $(-\infty,b_0)$. Similarly, $\mathcal{R}_\gamma$ is increasing on the right of $b_\gamma$ and thus $b$ has higher robust error.
\end{proof}

\begin{proof}[Proof of \Cref{cor:l2 attacks}]
We can characterize the robust errors based on $l_2$ attacks in a similar fashion to \eqref{eq:robust error of any intercept}. We notice that
\begin{align*}
\mathcal{R}_{\gamma}(f)=& \P(\exists\|\p\|_2 \leq \epsilon \text { s.t. } f(\x+\p) \neq y) = \max _{\|\p\|_2 \leq \gamma} \P(f(\x+\p) \neq y) \\=& \frac{1}{2} \P(f(\x+\bm\gamma_d/\sqrt{d}) \neq-1 \mid y=-1) +\frac{1}{2} \P(f(\x-\bm\gamma_d/\sqrt{d}) \neq+1 \mid y=+1) 
\end{align*}
In short, the same analysis is in place except $\gamma\to\gamma/\sqrt{d}$ when we switch from $l_\infty$ to $l_2$ attacks.
\end{proof}
\section{Ablation Studies}
\label{app: ablation}

\subsection{CelebA}
\vspace{-0.3cm}
\begin{figure}[!htb]
    \centering
    \includegraphics[width=0.38\linewidth]{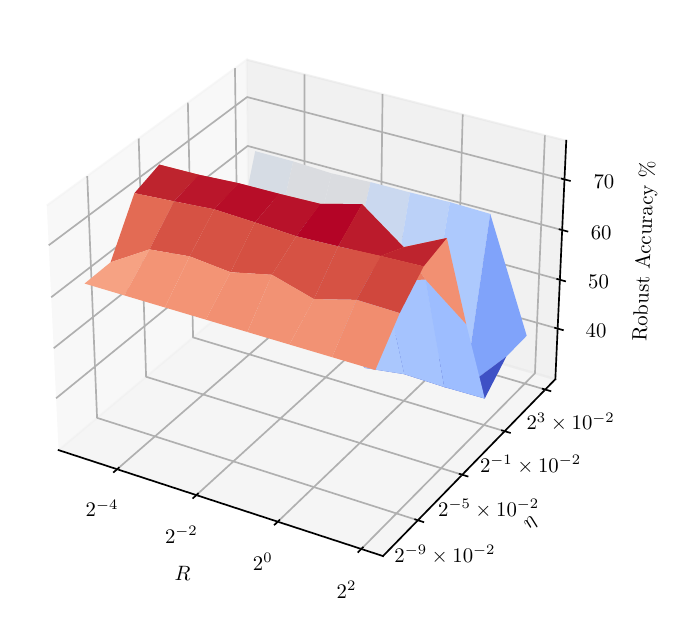}
    \includegraphics[width=0.38\linewidth]{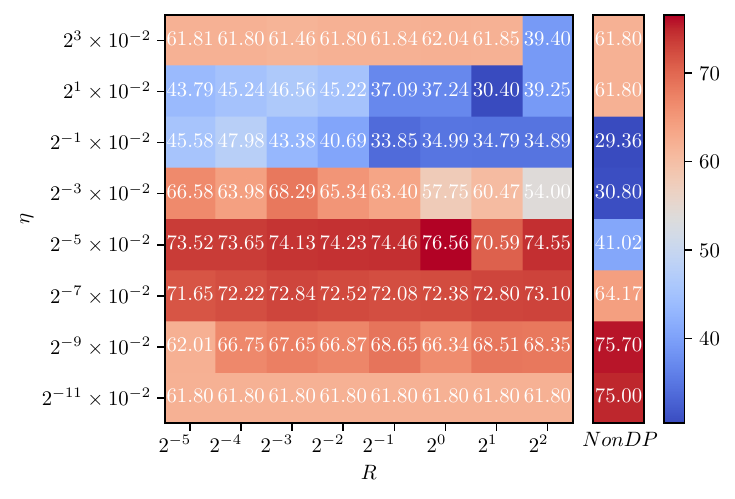} \\
    \vspace{-0.15cm}
    \includegraphics[width=0.4\linewidth]{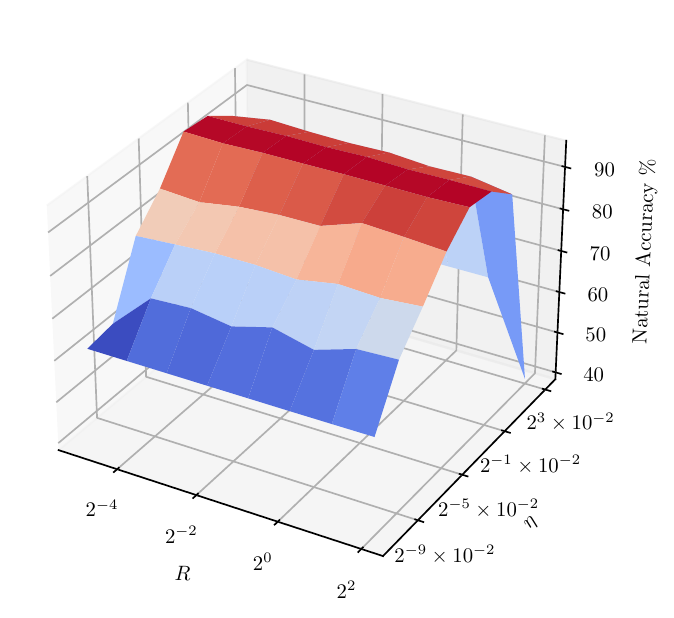}
    \includegraphics[width=0.4\linewidth]{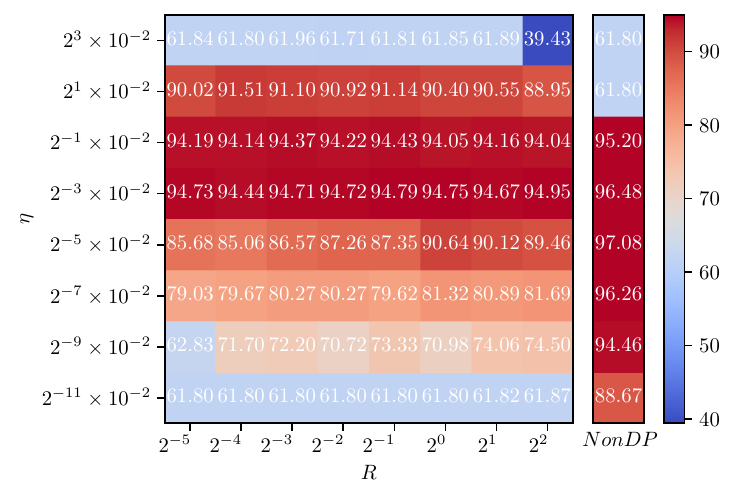}
    \vspace{-0.2cm}
    \caption{Robust and natural accuracy of $\eta$ and $R$ on CelebA with label `Male'. We train a 2-layer CNN 
    using DP-Adam and attack by $l_\infty(2/255)$ PGD attack. Same as in \Cref{fig:params-matter-celeba-adam-loglog-shu}. Here $\epsilon=2$, batch size $=512$, epochs $=10$.}
    \label{fig:params-matter-celeba-adam-loglog}
\end{figure}

\vspace{-0.4cm}
\begin{figure}[!htb]
\centering
\includegraphics[width=0.42\linewidth]{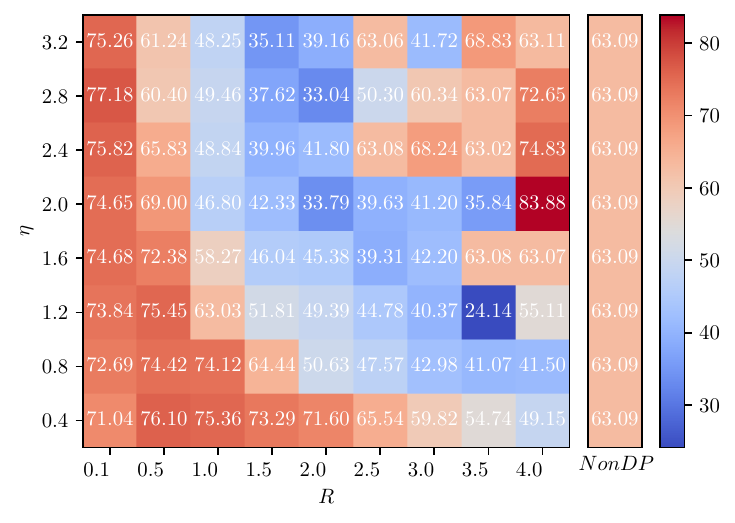}
\includegraphics[width=0.42\linewidth]{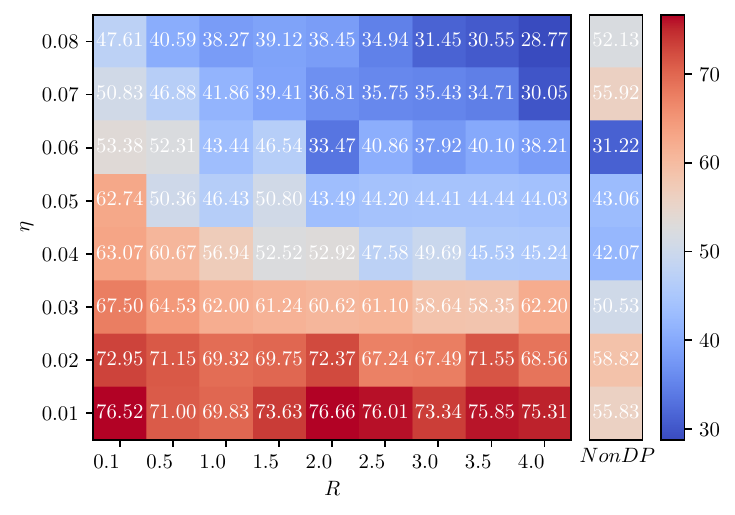} \\
\includegraphics[width=0.42\linewidth]{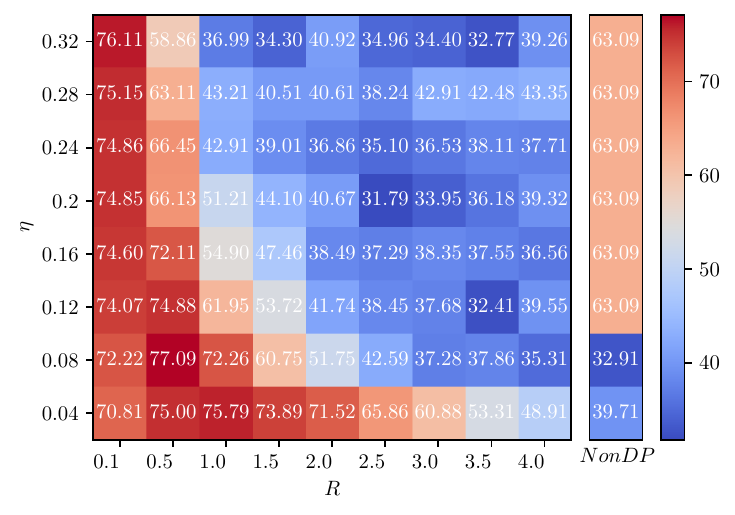}
\includegraphics[width=0.47\linewidth]{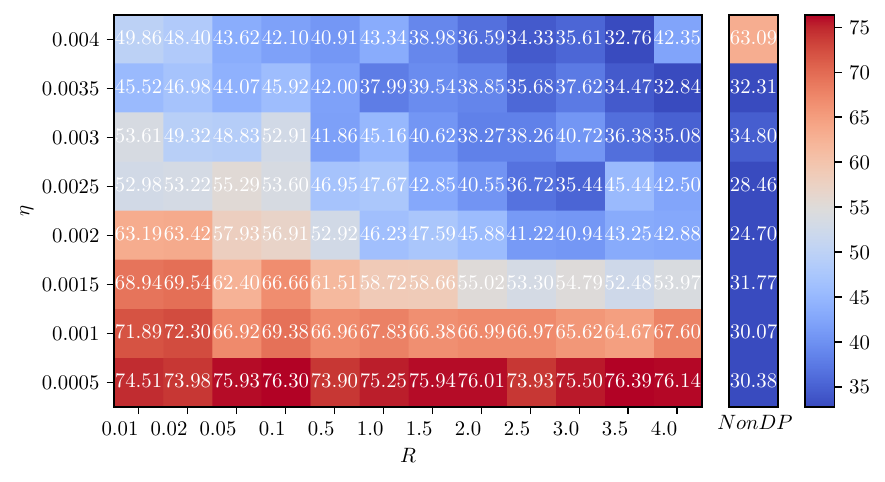}
\caption{Robust accuracy of CelebA with label `Male' under different optimizer, trained with a 2-layer CNN 
and attacked by $l_\infty(2/255)$ PGD attack. Top left: SGD. Top right: Adagrad. Bottom left: SGD momentum. Bottom right: Adam. Here $\epsilon=2$, batch size $=512$, epochs $=10$.}
\label{fig:optimizers CelebA}
\end{figure}

\begin{figure}[!htb]
    \centering
    \includegraphics[width=0.38\linewidth]{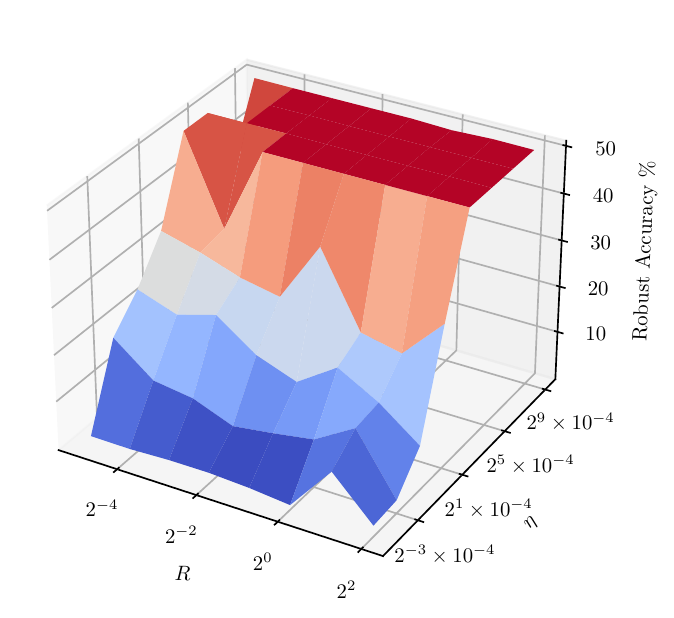}
    \includegraphics[width=0.38\linewidth]{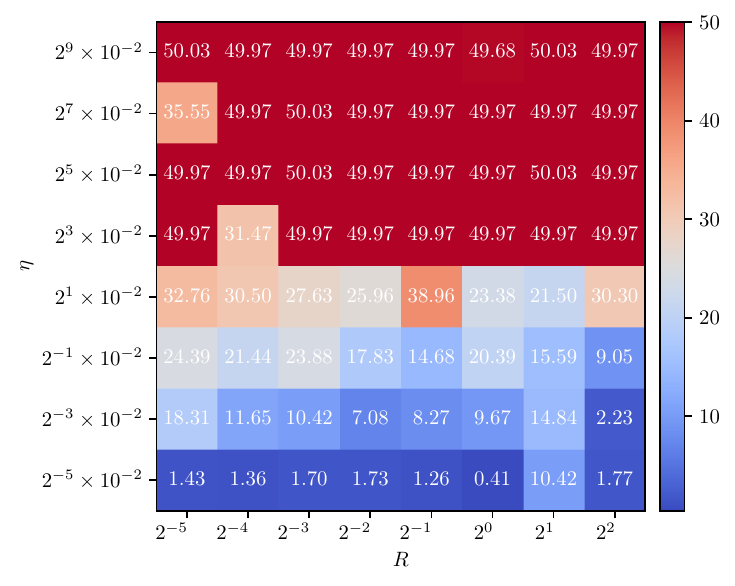} \\
\vspace{-0.15cm}
    \includegraphics[width=0.4\linewidth]{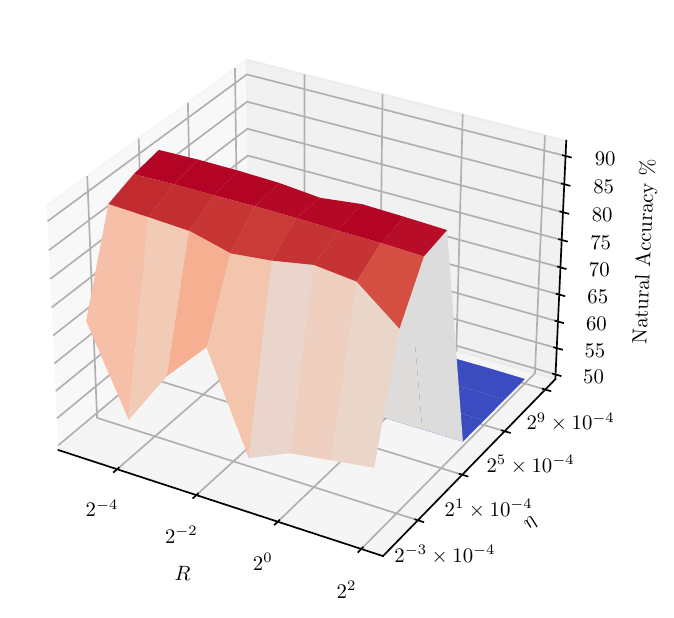}
    \includegraphics[width=0.4\linewidth]{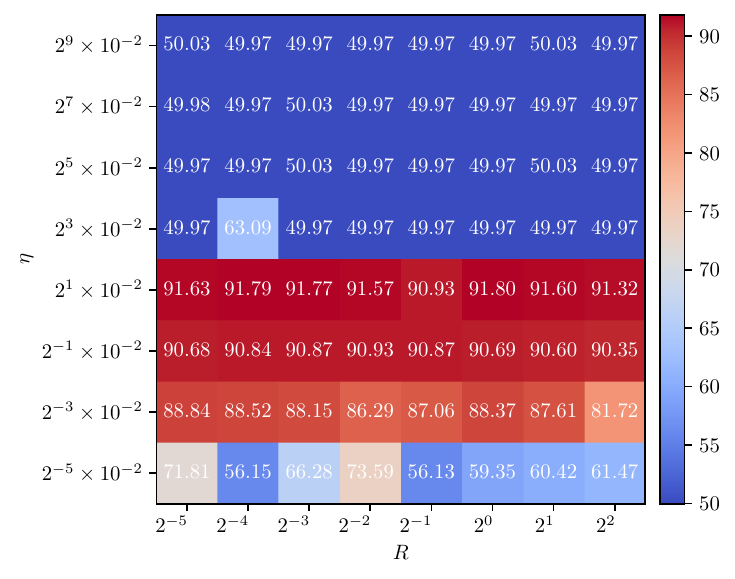}
\vspace{-0.2cm}
    \caption{Robust and natural accuracy of $\eta$ and $R$ on CelebA with label `Smiling'. We train ViT-tiny using DP-RMSprop and attack by $l_\infty(2/255)$ PGD attack. Here $\epsilon=2$, batch size $=1024$, epoch $=1$.}
    \label{fig:params-matter-celeba-adam-loglog22}
\end{figure}

\clearpage
\vspace{-0.2cm}
\subsection{CIFAR10}
\vspace{-0.2cm}
\begin{figure}[!htb]
    \centering
    \includegraphics[width=0.37\linewidth]{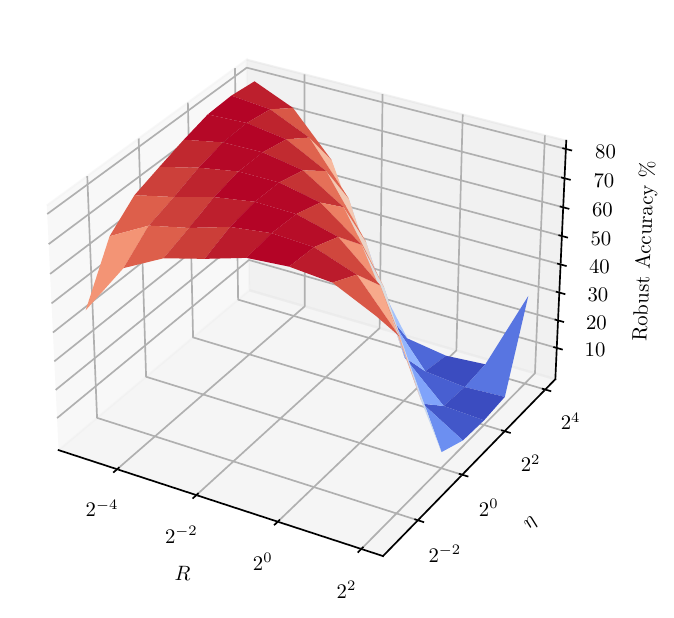}
    \includegraphics[width=0.37\linewidth]{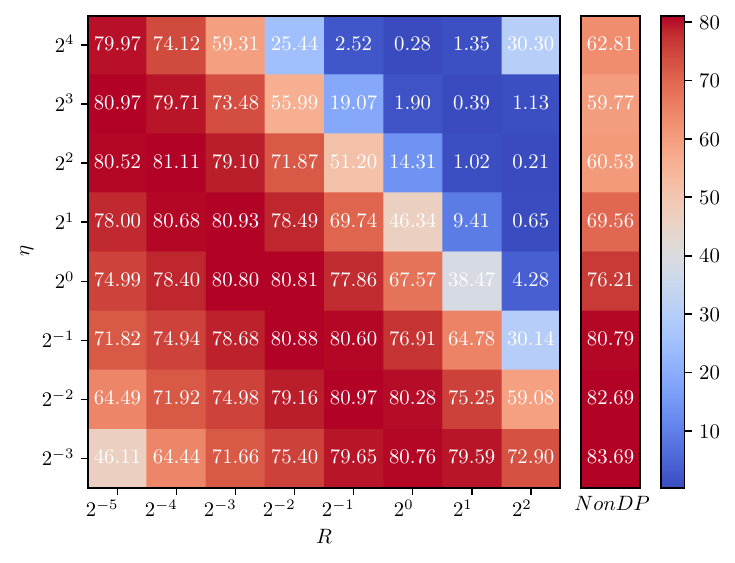} \\
    \vspace{-0.15cm}
    \includegraphics[width=0.37\linewidth]{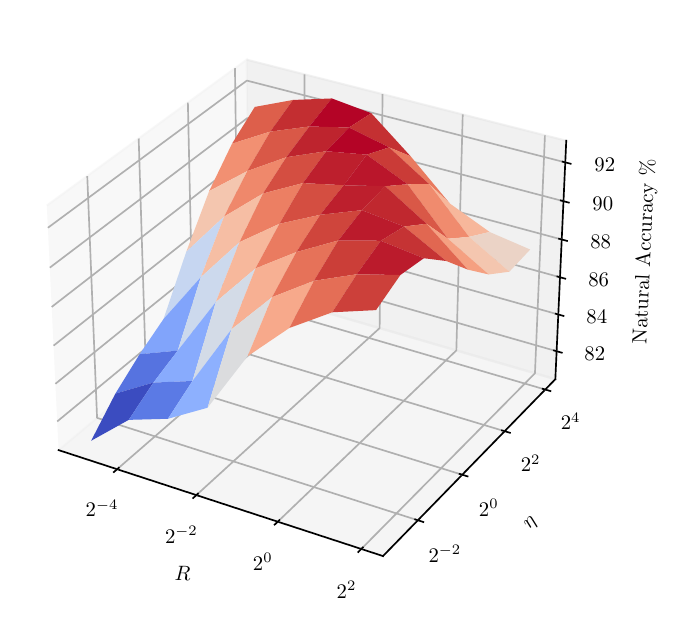}
    \includegraphics[width=0.37\linewidth]{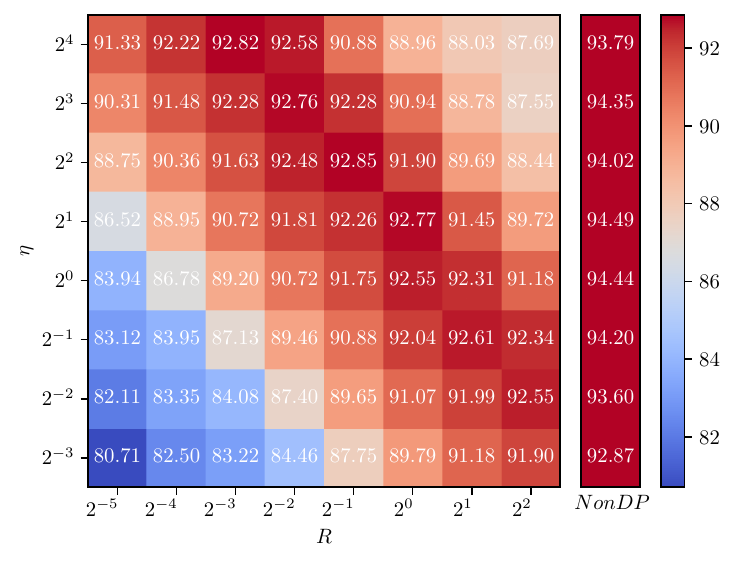}
\vspace{-0.4cm}
    \caption{Robust and clean accuracy of $\eta$ and $R$ on CIFAR10, transferred from SimCLRv2 pre-trained on unlabelled ImageNet. We use DP-SGD and attack by $l_\infty(2/255)$ PGD attack. Here $\epsilon=2$, batch size $=1024$, epochs $=50$.}
    \label{fig:params-matter-cifar10-simclr-sgd-loglog}
\end{figure}

\clearpage
\subsection{MNIST}
\vspace{-0.2cm}
\begin{figure}[!htb]
    \centering
    \includegraphics[width=0.38\linewidth]{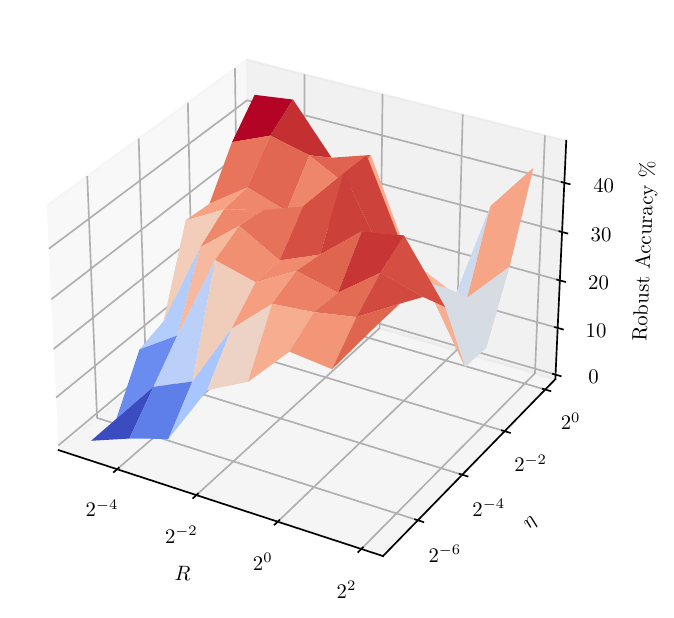}
    \includegraphics[width=0.38\linewidth]{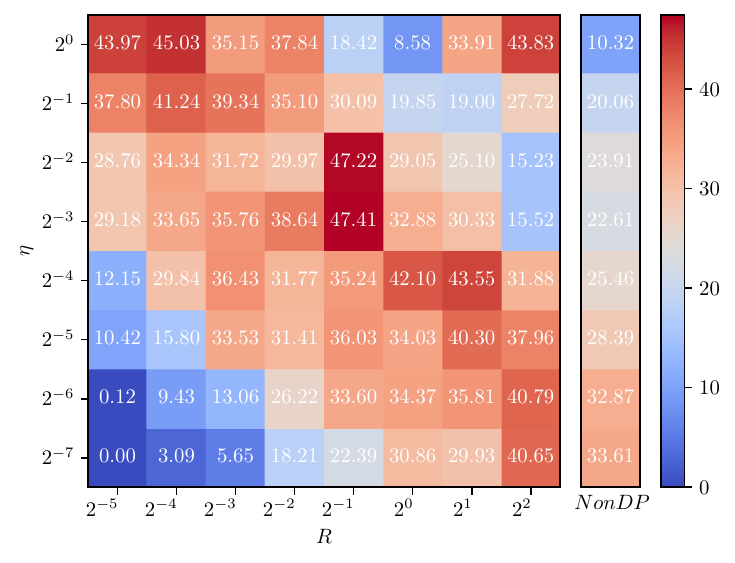} \\\vspace{-0.2cm}
    \includegraphics[width=0.38\linewidth]{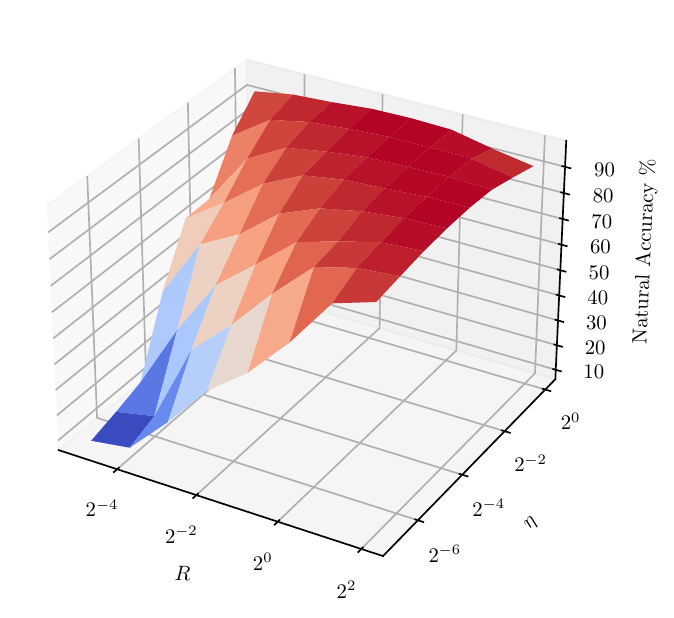}
    \includegraphics[width=0.38\linewidth]{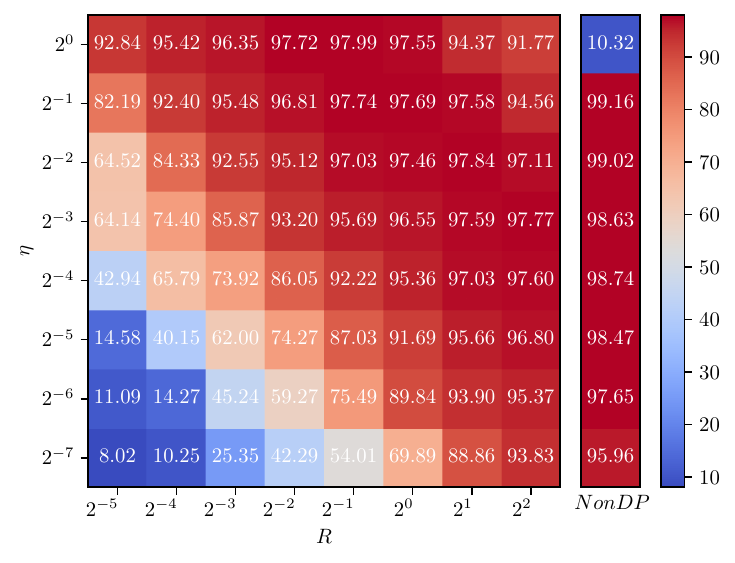}
\vspace{-0.4cm}
    \caption{Robust and clean accuracy of $\eta$ and $R$ on MNIST. We train the CNN from \citet{tramer2020differentially} using DP-SGD and attack by $l_\infty(32/255)$ PGD attack. Here $\epsilon=2$, batch size $=512$, epochs $=40$.} \label{fig:params-matter-mnist-lr0.8}
\end{figure}

\clearpage
\vspace{-0.2cm}
\subsection{Fashion MNIST}
\vspace{-0.2cm}
\begin{figure}[!htb]
    \centering
    \includegraphics[width=0.38\linewidth]{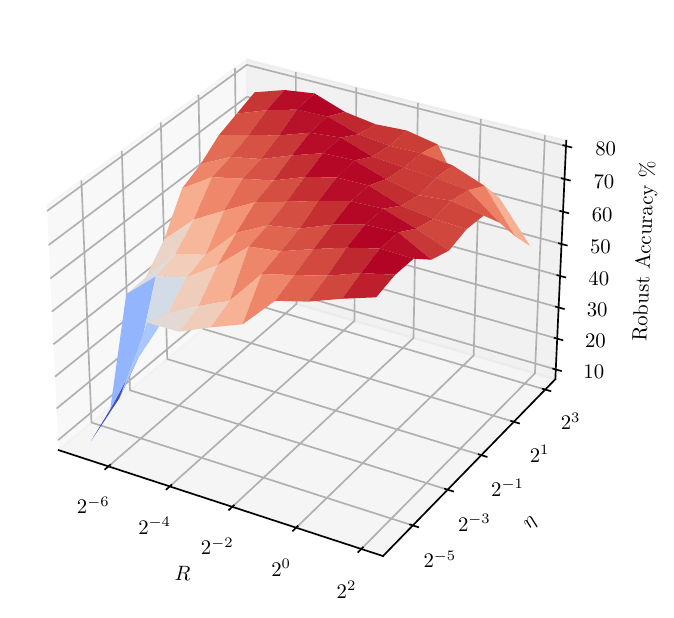}
    \includegraphics[width=0.38\linewidth]{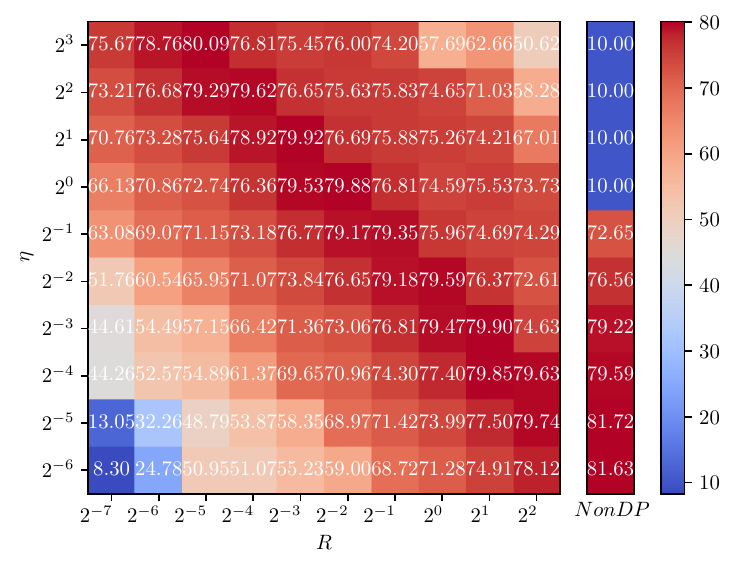} \\
    \vspace{-0.15cm}
\includegraphics[width=0.38\linewidth]{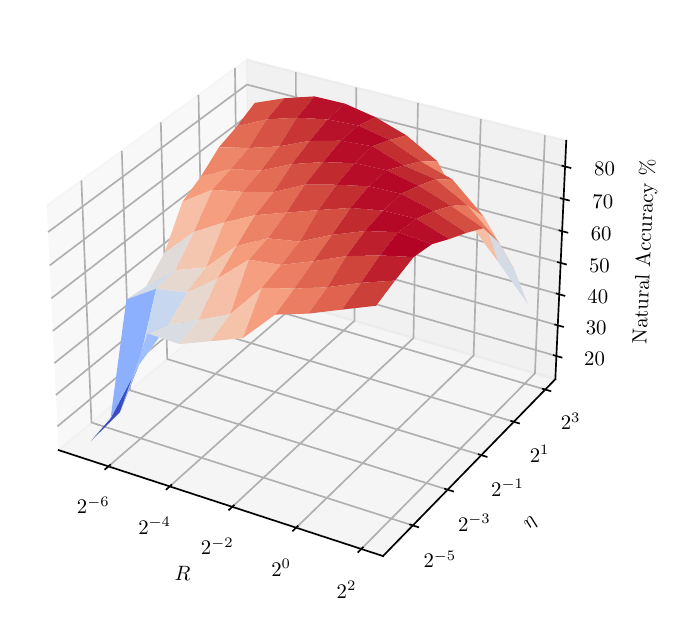}
    \includegraphics[width=0.38\linewidth]{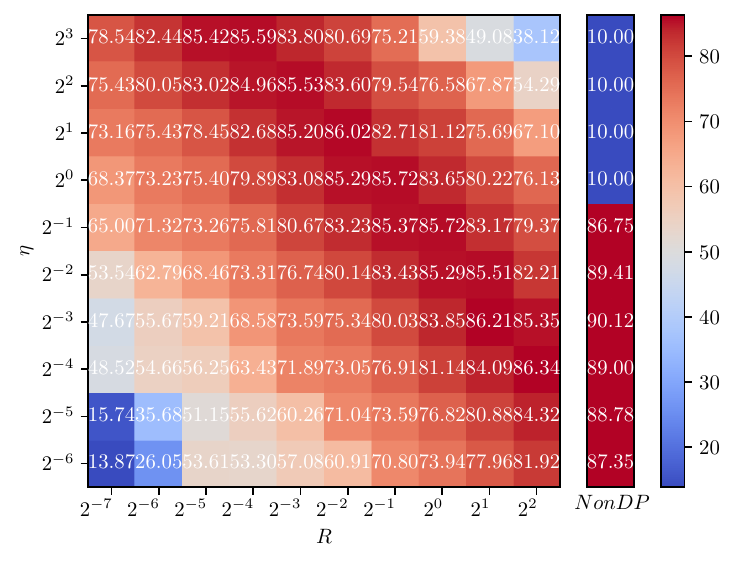}
    \vspace{-0.3cm}
    \caption{Robust and clean accuracy of $\eta$ and $R$ on Fashion MNIST. We train the CNN from \citet{tramer2020differentially} using DP-SGD and attack by $l_\infty(2/255)$ PGD attack. Here $\epsilon=2$, batch size $=2048$, epochs $=40$.}
    \label{fig:params-matter-fmnist-sgdmtm-loglog}
\end{figure}

\clearpage
\section{More tables}
\label{app:more tables}

\begin{table}[!htb]
    \centering
\begin{tabular}{|c|c|c|c|c|c||c|c|}
\hline & Natural & FGSM & BIM& PGD$_{\infty} $ & APGD$_{\infty} $ & PGD$_2 $ & APGD$_2 $  \\
\hline Non-DP & $94.55\%$ & $18.71\%$ & $15.97 \%$ & $15.96\%$ & $16.04\%$ & $35.95\%$ & $35.89\%$ \\
\hline DP , $\epsilon=2$ & $92.73\%$ & $10.35\%$ & $0.03\%$ & $0.03\%$ & $0.03\%$ &$12.76\%$ & $12.68\%$ \\
\hline DP , $\epsilon=4$ & $93.49\%$ & $30.10\%$ & $9.10\%$ & $9.09\%$ & $9.12\%$ &$40.97\%$ & $41.01\%$ \\
\hline DP , $\epsilon=8$ & $93.74\%$ & $31.86\%$ & $28.08\%$ & $28.09\%$ & $28.09\%$ &$54.53\%$ & $54.54\%$ \\
\hline
\end{tabular}
    \caption{Natural and robust accuracy of models transferred from unlabelled ImageNet pre-trained SIMCLRv2 on CIFAR10 under general adversarial attacks with $\gamma_{\infty}=4/255$ and $\gamma_{2}=0.5$. Attack steps are 20 if applicable. Model hyper-parameters are directly adopted from \citet{tramer2020differentially} for highest natural accuracy. DP models are trained using DP-SGD, $R=0.1$, $\eta_{DP}=4$, momentum $=0.9$, batch size $=1024$. Non-DP models are trained using SGD with the same hyper-parameters except $\eta_{non-DP}=0.4$.}
    \label{tab:multi-atk-handcrafted-cifar10}
\end{table}

\begin{table}[!htb]
    \centering
\begin{tabular}{|c|c|c|c|c|}
\hline& Non-DP & DP& DP& DP\\
attack magnitude& $ \epsilon=\infty$ & $ \epsilon=2 $&$ \epsilon=4 $&$ \epsilon=8 $ \\
\hline $ \gamma=0.0 $ & $ 99.24 \%$ & $ 98.01 \%$& $ 98.32 \%$& $ 98.50 \%$ \\
\hline $ \gamma=0.25 $ & $ 97.57 \%$ & $ 95.29 \%$ & $ 95.94 \%$& $ 96.65 \%$\\
\hline $ \gamma=0.5 $ & $ 93.32 \%$ & $ 90.28 \%$ & $ 91.71 \%$& $ 92.97 \%$\\
\hline $ \gamma=1.0 $ & $ 66.58 \% $ & $ 63.95 \%$& $ 73.32 \%$ & $ 77.08 \%$\\
\hline $ \gamma=2.0 $ & $ 36.28 \%$ & $ 39.88 \%$ & $ 51.48 \%$& $ 52.74 \%$\\
\hline
\end{tabular}
    \caption{Robust accuracy on MNIST under 20 steps $l_2$ PGD attack. Model hyper-parameters are directly adopted from \citet{tramer2020differentially} for highest natural accuracy. DP models are trained using DP-SGD, $R=0.1$, $\eta_{DP}=0.5$, momentum $=0.9$, batch size $=512$. Non-DP models are trained using SGD with the same hyper-parameters except $\eta_{non-DP}=0.05$.}
    \label{tab:l2-atk-handcrafted-mnist}
\end{table}

\begin{table}[!htb]
    \centering
\begin{tabular}{|c|c|c|c|c|}
\hline& Non-DP & DP& DP& DP\\
attack magnitude& $ \epsilon=\infty$ & $ \epsilon=2 $&$ \epsilon=4 $&$ \epsilon=8 $ \\
\hline $ \gamma=0.0 $ & $ 99.24 \%$ & $ 98.01 \%$& $ 98.32 \%$& $ 98.50 \%$ \\
\hline $ \gamma=2/255 $ & $ 98.73 \%$ & $ 97.12 \%$ & $ 97.43 \%$& $ 97.84 \%$\\
\hline $ \gamma=4/255 $ & $ 97.88 \%$ & $ 95.78 \%$ & $ 96.32 \%$& $ 97.13 \%$\\
\hline $ \gamma=8/255 $ & $ 95.32 \% $ & $ 92.31 \%$& $ 93.51 \%$ & $ 94.74 \%$\\
\hline $ \gamma=16/255 $ & $ 82.06 \%$ & $ 77.67 \%$ & $ 80.28 \%$& $ 85.82 \%$\\
\hline
\end{tabular}
    \caption{Robust accuracy on MNIST under 20 steps $l_{\infty}$ PGD attack. Model hyper-parameters are directly adopted from \citet{tramer2020differentially} for highest natural accuracy. DP models are trained using DP-SGD, $R=0.1$, $\eta_{DP}=0.5$, momentum $=0.9$, batch size $=512$. Non-DP models are trained using SGD with the same hyper-parameters except $\eta_{non-DP}=0.05$.}
    \label{tab:linf-atk-handcrafted-mnist}
\end{table}

\begin{table}[!htb]
    \centering
\begin{tabular}{|c|c|c|c|c|c||c|c|}
\hline & Natural & FGSM & BIM& PGD$_{\infty} $ & APGD$_{\infty} $ & PGD$_2 $ & APGD$_2 $ \\
\hline Non-DP & $99.24\%$ & $97.92\%$ & $97.88\%$ & $97.88\%$ & $97.77\%$ & $93.32\%$ & $93.27\%$ \\
\hline DP , $\epsilon=2$ & $98.01\%$ & $95.89\%$ & $95.80\%$ & $95.79\%$ & $95.63\%$ &$90.28\%$ & $90.15\%$ \\
\hline DP , $\epsilon=4$ & $98.32\%$ & $96.45\%$ & $96.32\%$ & $96.33\%$ & $96.27\%$ &$91.71\%$ & $91.68\%$ \\
\hline DP , $\epsilon=8$ & $98.50\%$ & $97.19\%$ & $97.15\%$ & $97.15\%$ & $97.06\%$ &$92.97\%$ & $92.94\%$ \\
\hline
\end{tabular}
    \caption{Natural and robust accuracy of CNN models on MNIST under general adversarial attacks with $\gamma_{\infty}=4/255$ and $\gamma_{2}=0.5$. Attack steps are 20 if applicable. Model hyper-parameters are directly adopted from \citet{tramer2020differentially} for highest natural accuracy. DP models are trained using DP-SGD, $R=0.1$, $\eta_{DP}=0.5$, momentum $=0.9$, batch size $=512$. Non-DP models are trained using SGD with the same hyper-parameters except $\eta_{non-DP}=0.05$.}
    \label{tab:multi-atk-handcrafted-mnist}
\end{table}

\begin{table}[!htb]
    \centering
\begin{tabular}{|c|c|c|c|c|}
\hline& Non-DP & DP& DP& DP\\
attack magnitude& $ \epsilon=\infty$ & $ \epsilon=2 $&$ \epsilon=4 $&$ \epsilon=8 $ \\
\hline $ \gamma=0.0 $ & $ 89.75 \%$ & $ 85.95 \%$& $ 86.60 \%$& $ 86.74 \%$ \\
\hline $ \gamma=0.25 $ & $ 57.37 \%$ & $ 69.24 \%$ & $ 72.93 \%$& $ 75.35 \%$\\
\hline $ \gamma=0.5 $ & $ 25.21 \%$ & $ 46.09 \%$ & $ 54.30 \%$& $ 59.23 \%$\\
\hline $ \gamma=1.0 $ & $ 7.87 \% $ & $ 16.77 \%$& $ 25.95 \%$ & $ 29.08 \%$\\
\hline $ \gamma=2.0 $ & $ 7.47 \%$ & $ 11.77 \%$ & $ 16.85 \%$& $ 17.00 \%$\\
\hline
\end{tabular}
    \caption{Robust accuracy on Fashion MNIST under 20 steps $l_{2}$ PGD attack. Model hyper-parameters are directly adopted from \citet{tramer2020differentially} for highest natural accuracy. DP models are trained using DP-SGD, $R=0.1$, $\eta_{DP}=4$, momentum $=0.9$, batch size $=2048$. Non-DP models are trained using SGD with the same hyper-parameters except $\eta_{non-DP}=0.4$.}
    \label{tab:l2-atk-handcrafted-fmnist}
\end{table}

\begin{table}[!htb]
    \centering
\begin{tabular}{|c|c|c|c|c|}
\hline& Non-DP & DP& DP& DP\\
attack magnitude& $ \epsilon=\infty$ & $ \epsilon=2 $&$ \epsilon=4 $&$ \epsilon=8 $ \\
\hline $ \gamma=0.0 $ & $ 89.75 \%$ & $ 85.95 \%$& $ 86.60 \%$& $ 86.74 \%$ \\
\hline $ \gamma=2/255 $ & $ 76.19 \%$ & $ 78.29 \%$ & $ 79.84 \%$& $ 81.47 \%$\\
\hline $ \gamma=4/255 $ & $ 64.46 \%$ & $ 69.75 \%$ & $ 72.60 \%$& $ 74.72 \%$\\
\hline $ \gamma=8/255 $ & $ 47.24 \% $ & $ 54.62 \%$& $ 57.87 \%$ & $ 60.52 \%$\\
\hline $ \gamma=16/255 $ & $ 23.26 \%$ & $ 28.51 \%$ & $ 31.68 \%$& $ 30.90 \%$\\
\hline
\end{tabular}
    \caption{Robust accuracy on Fashion MNIST under 20 steps $l_{\infty}$ PGD attack. Model hyper-parameters are directly adopted from \citet{tramer2020differentially} for highest natural accuracy. DP models are trained using DP-SGD, $R=0.1$, $\eta_{DP}=4$, momentum $=0.9$, batch size $=2048$. Non-DP models are trained using SGD with the same hyper-parameters except $\eta_{non-DP}=0.4$.}
    \label{tab:linf-atk-handcrafted-fmnist}
\end{table}

\begin{table}[!htb]
    \centering
\begin{tabular}{|c|c|c|c|c|c||c|c|}
\hline & Natural & FGSM & BIM& PGD$_{\infty} $ & APGD$_{\infty} $ & PGD$_2 $ & APGD$_2 $ \\
\hline Non-DP & $89.75\%$ & $70.41\%$ & $64.56\%$ & $64.44\%$ & $53.41\%$ & $25.21\%$ & $23.13\%$ \\
\hline DP , $\epsilon=2$ & $85.95\%$ & $72.11\%$ & $69.76\%$ & $69.71\%$ & $67.13\%$ &$46.09\%$ & $45.41\%$ \\
\hline DP , $\epsilon=4$ & $86.60\%$ & $73.67\%$ & $72.68\%$ & $72.69\%$ & $70.84\%$ &$54.30\%$ & $53.92\%$ \\
\hline DP , $\epsilon=8$ & $86.74\%$ & $75.45\%$ & $74.75\%$ & $74.74\%$ & $73.71\%$ &$59.23\%$ & $58.98\%$\\
\hline
\end{tabular}
    \caption{Natural and robust accuracy of CNN models on Fashion MNIST under general adversarial attacks with $\gamma_{\infty}=4/255$ and $\gamma_{2}=0.5$. Attack steps are 20 if applicable. Model hyper-parameters are directly adopted from \citet{tramer2020differentially} for highest natural accuracy. DP models are trained using DP-SGD, $R=0.1$, $\eta_{DP}=4$, momentum $=0.9$, batch size $=2048$. Non-DP models are trained using SGD with the same hyper-parameters except $\eta_{non-DP}=0.4$.}
\end{table}

\begin{table}[!htb]
    \centering
\begin{tabular}{|c|c|c|c|c|c||c|c|}
\hline & Natural & FGSM & BIM& PGD$_{\infty} $ & APGD$_{\infty} $ & PGD$_2 $ & APGD$_2 $ \\
\hline Non-DP & $94.29\%$ & $14.48\%$ & $12.02 \%$ & $12.00\%$ & $12.03\%$ & $31.36\%$ & $31.28\%$ \\
\hline DP , $\epsilon=2$ & $92.73\%$ & $15.70\%$ & $1.59\%$ & $1.61\%$ & $1.62\%$ &$28.05\%$ & $28.06\%$ \\
\hline DP , $\epsilon=4$ & $93.49\%$ & $30.89\%$ & $5.23\%$ & $5.27\%$ & $5.25\%$ &$35.96\%$ & $35.98\%$ \\
\hline DP , $\epsilon=8$ & $93.74\%$ & $9.66\%$ & $4.30\%$ & $4.29\%$ & $4.31\%$ &$33.21\%$ & $33.23\%$ \\
\hline
\end{tabular}
    \caption{Natural and robust accuracy of models transferred from unlabelled ImageNet pre-trained SIMCLRv2 on CIFAR10 under general adversarial attacks with $\gamma_{\infty}=4/255$ and $\gamma_{2}=0.5$. Attack steps are 20 if applicable. Model in each row is the most accurate model obtained by simple grid search: Non-DP: $\eta=0.5$; $\mathrm{DP}_{\epsilon=2}$: $\eta=1.0, R=0.25$; $\mathrm{DP}_{\epsilon=4}$: $\eta=8, R=0.0625$, $\mathrm{DP}_{\epsilon=8}$: $\eta=0.5, R=1.0$. All models are trained using SGD or DP-SGD, momentum $=0.9$ and batch size $=1024$.}
    \label{tab:multi-atk-cifar10}
\end{table}

\newpage
\section{Hyper-parameter setup}
\label{app:hyperparameters}
In \Cref{tab:cifar with madry}, SimCLRv2 models are pre-trained on unlabelled ImageNet and fine-tuned on CIFAR10. \textit{Natural} models are directly adopted from \citet{tramer2020differentially} for highest natural accuracy, where optimizer is DP-SGD and SGD, $R=0.1$, $\eta_{DP}=4$, $\eta_{non-DP}=0.4$, momentum $=0.9$, batch size $=1024$. \textit{Robust} models are obtained by grid search over $\eta$ and $R$ against $l_\infty(2/255)$, where Non-DP: $\eta=0.0625$; $\mathrm{DP}_{\epsilon=2}$: $\eta=4, R=0.0625$; $\mathrm{DP}_{\epsilon=4}$: $\eta=0.5, R=0.0625$, $\mathrm{DP}_{\epsilon=8}$: $\eta=0.125, R=0.25$. Similarly, adversarial training models are obtained by grid search over $\eta$ and a fixed $R=0.0625$, where Non-DP: $\eta=0.25$; $\mathrm{DP}_{\epsilon=2}$: $\eta=0.5$; $\mathrm{DP}_{\epsilon=4}$: $\eta=1$, $\mathrm{DP}_{\epsilon=8}$: $\eta=2$. Other settings are the same as the \textit{natural} ones. Adversarial attack is $l_\infty$, 20 steps, alpha $=0.1$.

In \Cref{tab:cifar with madry l2}, SimCLRv2 models are pre-trained on unlabelled ImageNet and fine-tuned on CIFAR10. \textit{Natural} models are directly adopted from \citet{tramer2020differentially} for highest natural accuracy, where optimizer is DP-SGD and SGD, $R=0.1$, $\eta_{DP}=4$, $\eta_{non-DP}=0.4$, momentum $=0.9$, batch size $=1024$. \textit{Robust} models are obtained by grid search over $\eta$ and $R$ against $l_2(0.25)$, where Non-DP: $\eta=0.0625$; $\mathrm{DP}_{\epsilon=2}$: $\eta=0.0625, R=0.25$; $\mathrm{DP}_{\epsilon=4}$: $\eta=0.5, R=0.0625$, $\mathrm{DP}_{\epsilon=8}$: $\eta=0.125, R=0.25$. Similarly, adversarial training models are obtained by grid search over $\eta$ and a fixed $R=0.0625$, where Non-DP: $\eta=0.0625$; $\mathrm{DP}_{\epsilon=2}$: $\eta=1$; $\mathrm{DP}_{\epsilon=4}$: $\eta=1$, $\mathrm{DP}_{\epsilon=8}$: $\eta=2$. Other settings are the same as the \textit{natural} ones. Adversarial attack is $l_2$, 20 steps, alpha $=0.1$.

In \Cref{fig:rob_acc_vs_clean_acc_cifar10_simclr}, models are SimCLRv2 pre-trained on unlabelled ImageNet and fine-tuned on CIFAR10 using DP-SGD, with $\epsilon=8$, batch size $=1024$. Adversarial attack is $l_{\infty}$ PGD, $\gamma=4/255$, alpha=0.1.

In \Cref{tab:multi-atk-cifar10_rob}, models the same as in \Cref{tab:cifar with madry} with \textit{robust} hyper-parameters, where optimizer is DP-SGD and SGD, momentum $=0.9$, batch size $=1024$, Non-DP: $\eta=0.0625$; $\mathrm{DP}_{\epsilon=2}$: $\eta=4, R=0.0625$; $\mathrm{DP}_{\epsilon=4}$: $\eta=0.5, R=0.0625$, $\mathrm{DP}_{\epsilon=8}$: $\eta=0.125, R=0.25$. Adversarial attack steps $=20$, alpha $=0.1$ if applicable.

In \Cref{tab:resnet18 and vit}, models are ResNet18 and ViT-tiny trained on CelebA, label Smiling. Images are resized to $224\times 224$. Optimizer is DP-RMSprop with epochs $=5$, batch size $=1024$, $\eta=0.0002$, $R=0.1$, delta=5e-6. Adversarial attack is $l_{\infty}$ PGD, 20 steps, alpha $=1/255$.

In \Cref{tab:resnet18 general attacks}, models are ResNet18 as in \Cref{tab:resnet18 and vit}, trained on CelebA, label 'Smiling'. Images are resized to $224\times 224$. Optimizer is DP-RMSprop with epochs $=5$, batch size $=1024$, $\eta=0.0002$, $R=0.1$, delta=5e-6. Adversarial attack is $l_{\infty}(2/255)$ with $\text{alpha}_{\infty}=1/255$ and $l_2(0.25)$ with $\text{alpha}_{2}=0.2$, 20 steps, if applicable.

In \Cref{fig:params-matter-celeba-adam-loglog-shu}, models are 2-layer CNN
trained on CelebA label `Male' using DP-Adam, where $\epsilon=2$, batch size $=512$, epochs $=10$. Adversarial attack is $l_\infty(2/255)$ PGD, 20 steps, alpha $=0.1$. 

\section{Extra}
\label{app:RMSprop}
\subsection{Robust and accuracy landscapes of DP optimizers}

\begin{figure}[!htb]
    \vspace{-0.5cm}
    \centering
    \includegraphics[width=0.4\linewidth]{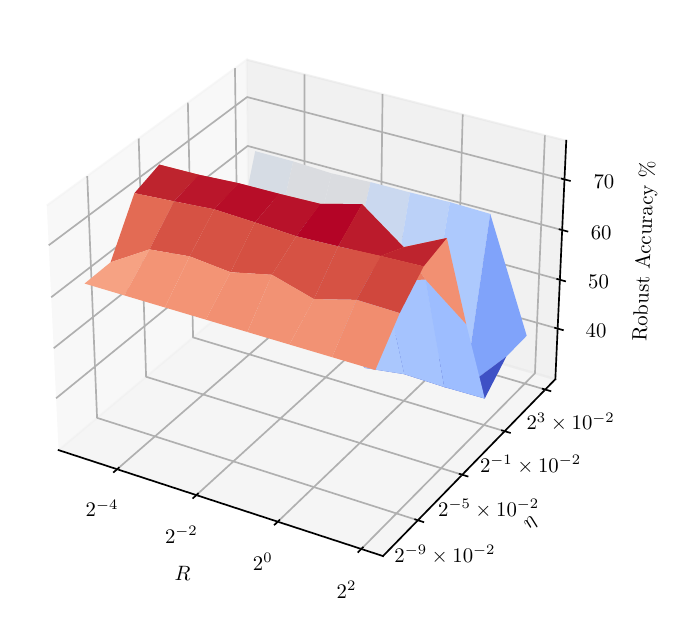}
    \includegraphics[width=0.4\linewidth]{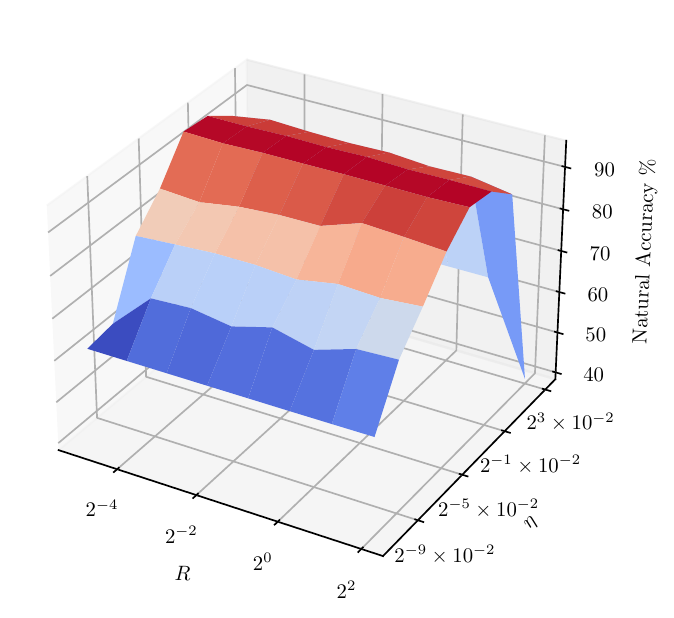}
    \vspace{-0.2cm}
    \caption{Robust and natural accuracy by $\eta$ and $R$ on CelebA with label `Male'. We train a simple CNN with DP-Adam and test under 20 steps of $l_\infty(2/255)$ PGD attack. See details in \Cref{app:hyperparameters}.
    }
    \label{fig:params-matter-celeba-adam-loglog-shu}
\end{figure}

We note that the diagonal pattern of the accuracy landscapes observed in \Cref{fig:params-matter-celeba-simclr-sgdmtm-loglog} (CIFAR10 \& DP-Heavyball), \Cref{fig:params-matter-cifar10-simclr-sgd-loglog} (CIFAR10 \& DP-SGD), \Cref{fig:params-matter-mnist-lr0.8} (MNIST \& DP-SGD) and \Cref{fig:params-matter-fmnist-sgdmtm-loglog} (Fashion MNIST \& DP-SGD) is not universal. For example, in \Cref{fig:params-matter-celeba-adam-loglog}, we show that adaptive optimizers are much less sensitive to the clipping norm $R$, as the landscapes are characterized by the row-wise pattern instead of the diagonal pattern. This pattern is particularly obvious in the small $R$ regime, where the robust and natural accuracy are high (see right panel in \Cref{fig:params-matter-celeba-adam-loglog}.).

To rigorously analyze the insensitivity to the clipping norm in a simplified manner, we take the RMSprop\citep{tieleman2012lecture} as an example, similar to the analysis in \citep{bu2022automatic} on DP-Adam. When $R$ is sufficiently small, the private gradient in \eqref{eq:private grad} becomes
\begin{align*}
  \tilde{\g}_t &= \sum_i \frac{\g_t(\x_i)}{\max(1, ||\g_t(\x_i)||_2/R)} + \sigma R\mathcal{N}(0, I) 
  = \sum_i \frac{\g_t(\x_i)}{||\g_t(\x_i)||_2/R} + \sigma R\mathcal{N}(0, I)
  \\
  &=R\cdot\left(\sum_i \frac{\g_t(\x_i)}{||\g_t(\x_i)||_2} + \sigma\mathcal{N}(0, I)\right) 
  :=R\cdot\hat{\g}_t.
\end{align*}

With private gradient $\tilde\g_t$, DP-RMSprop updates the parameters $\bm\theta_{t}$ by
\begin{align}
  \bm\theta_{t} =\bm\theta_{t-1} + \eta_t \frac{\tilde{\g}_t}{\sqrt{\tilde{\bm v_t}}}, \label{eq:dp-rmsprop theta}
\end{align}
where $\tilde{\bm v}$ is the squared average of $\tilde{\g_t}$, written as
\begin{align}
  \begin{split}
  \tilde{\bm v_t} &= \alpha \tilde{\bm v}_{t-1} + (1-\alpha) \tilde{\g}^2_t
  = \sum_{s}^{t} (1-\alpha) \alpha^{t-s} \tilde{\g}^2_s 
  = R^2 \cdot \sum_{s}^{t} (1-\alpha) \alpha^{t-s} \hat{\g}^2_s \end{split}
\label{eq:tilde v}
\end{align}

\vspace{-0.2cm}
Substitute \eqref{eq:tilde v} into \eqref{eq:dp-rmsprop theta}, we obtain an updating rule that is independent of the clipping norm $R$,
\begin{align*}
  \bm \theta_{t} &=\bm \theta_{t-1} + \eta_t \frac{R\cdot\hat{\g}_t}{\sqrt{R^2 \cdot \sum_{s}^{t} (1-\alpha) \alpha^{t-s} \hat{\g}^2_s}} =\bm\theta_{t-1} + \eta_t \frac{\hat{\g}_t}{\sqrt{\sum_{s}^{t} (1-\alpha) \alpha^{t-s} \hat{\g}^2_s}}.
\end{align*}


As a result, DP optimizers can have fundamentally different landscapes with respect to the hyper-parameters $(R,\eta)$, which in turn may affect the accuracy and robustness as illustrated in \Cref{fig:optimizers CelebA}.


\end{document}